\documentclass{article}

\usepackage{microtype}
\usepackage{graphicx}
\usepackage{subfigure}
\usepackage{booktabs} 

\usepackage{hyperref}

\usepackage[accepted]{icml2024}

\usepackage{amsmath}
\usepackage{amssymb}
\usepackage{mathtools}
\usepackage{amsthm}

\usepackage[capitalize,noabbrev]{cleveref}

\theoremstyle{plain}
\newtheorem{theorem}{Theorem}[section]
\newtheorem{proposition}[theorem]{Proposition}

\theoremstyle{definition}
\newtheorem{definition}[theorem]{Definition}

\theoremstyle{remark}

\usepackage[textsize=tiny]{todonotes}

\usepackage{adjustbox}
\usepackage{multirow}
\usepackage{booktabs}
\usepackage{xspace}
\usepackage{wrapfig}
\newcommand{\ie}{i.e., }
\newcommand{\eg}{e.g., }

\definecolor{first}{HTML}{ef8700}
\definecolor{second}{HTML}{785EF0}
\definecolor{seqcolor}{HTML}{d55e00} 
\definecolor{pascalcolor}{HTML}{029e73} 
\definecolor{linkcolor}{HTML}{0173b2} 
\definecolor{fourth}{HTML}{FFC107}

\newcommand{\one}[1]{\textcolor{first}{\bf#1}}

\newcommand{\adgn}{CTAN\xspace}
\newcommand{\fulladgn}{\underline{c}ontinuous-\underline{t}ime graph \underline{a}nti-symmetric \underline{n}etwork\xspace}
\newcommand{\myparagraph}[1]{\noindent \textbf{#1.}}

\usepackage{multirow}
\usepackage{soul}

\icmltitlerunning{Long Range Propagation on Continuous-Time Dynamic Graphs}

\begin{document}

\twocolumn[
\icmltitle{Long Range Propagation on Continuous-Time Dynamic Graphs}



\icmlsetsymbol{equal}{*}

\begin{icmlauthorlist}
\icmlauthor{Alessio Gravina}{equal,pisa}
\icmlauthor{Giulio Lovisotto}{equal,huawei}
\icmlauthor{Claudio Gallicchio}{pisa}
\icmlauthor{Davide Bacciu}{pisa}
\icmlauthor{Claas Grohnfeldt}{huawei}
\end{icmlauthorlist}

\icmlaffiliation{pisa}{Department of Computer Science, University of Pisa, Pisa, Italy}
\icmlaffiliation{huawei}{Huawei Technologies, Munich, Germany}

\icmlcorrespondingauthor{Alessio Gravina}{alessio.gravina@phd.unipi.it}

\icmlkeywords{Machine Learning, ICML, deep graph network, graph neural network, long range interactions, continuous time dynamic graphs, dynamic graphs, temporal graphs, ordinary differential equations}

\vskip 0.3in
]



\printAffiliationsAndNotice{\icmlEqualContribution} 

\begin{abstract}
Learning Continuous-Time Dynamic Graphs (C-TDGs) requires accurately modeling spatio-temporal information on streams of irregularly sampled events.
While many methods have been proposed recently, we find that most message passing-, recurrent- or self-attention-based methods perform poorly on \textit{long-range} tasks. 
These tasks require correlating information that occurred ``far'' away from the current event, either spatially (higher-order node information) or along the time dimension (events occurred in the past).
To address long-range dependencies, we introduce \fulladgn (CTAN).
Grounded within the ordinary differential equations framework, our method is designed for efficient propagation of information.
In this paper, we show how CTAN's (i) long-range modeling capabilities are substantiated by theoretical findings and how  (ii) its empirical performance on synthetic long-range benchmarks 
and real-world benchmarks is superior to other methods.
Our results motivate CTAN's ability to propagate  long-range information in C-TDGs as well as the inclusion of long-range tasks as part of temporal graph models evaluation.


\end{abstract}

\section{Introduction}
Graphs are a highly expressive abstraction for modelling entities and their relations, \eg molecular structures, recommender systems, or traffic networks~\citep{social_network,google_maps,gravina_schizophrenia, errica_hidden_2023,Cini_Marisca_Bianchi_Alippi_2023, gravina_covid}.
Deep Graph Networks (DGNs)~\citep{BACCIU2020203, GNNsurvey} have lately emerged as a family of deep learning models that can effectively process and learn such structured information. 
While most of the proposed DGNs have been designed for static graphs, many real-world scenarios are inherently \textit{dynamic} in nature.
Examples include the continual activities and interactions between members of social as well as communication networks, recurrent purchases by users on e-commerce platforms, or evolving interactions of processes with files in an operating system.
A number of works investigated models that can process the temporal dimension of a dynamic graph~\citep{dynamicgraph_survey, gravina2023deep}, with recent interest in graphs defined through irregularly sampled event streams, known as Continuous-Time Dynamic Graphs (C-TDGs).
%
However, dynamic methods, which are based on static DGNs and recurrent neural networks (RNNs) as backbone architectures, often retain the limitations of their core components. Specifically, static DGNs suffer from the \textit{over-squashing} phenomenon~\citep{bottleneck}, which prevents the final network to learn and propagate long range information~\citep{gravina2023antisymmetric}. Similarly, RNNs often face similar challenges in propagating long-term dependencies, as evidenced by \citet{AntisymmetricRNN}, mainly due to exploding or vanishing gradients. With growing evidence from the static and dynamic case~\citep{dwivedi2022long, yu2023towards} that long-range dependencies are necessary for effective learning, the ability to learn properties beyond 
the event's temporal and spatial locality remains an open challenge in the C-TDG domain.  

In this paper, we propose the \fulladgn (\adgn), a framework for learning of C-TDGs with \emph{scalable} long range propagation of information, thanks to properties inherited from stable and non-dissipative ordinary differential equations (ODEs).
We establish theoretical conditions for achieving stability and non-dissipation in the \adgn ODE by employing anti-symmetric weight matrices, which is the key factor for modeling long-range spatio-temporal interactions.
The \adgn layer is derived from the forward Euler discretization of the designed differential equation. 
The formulation of \adgn allows scaling the radius of propagation of information depending on the number of discretization steps, \ie the number of layers in the final architecture.
Remarkably, even with a limited number of layers, the non-dissipative behavior enables the transmission of information for a past event as new events occur, since node states are used to efficiently retain and propagate historical information. 
This mechanism permits scaling the single event propagation to cover a larger portion of the C-TDG.
The general formulation of the node update state function allows the implementation of the more appropriate dynamic to the problem at hand. Specifically, it allows the inclusion of static DGN dynamics, thus reinterpreting current state-of-the-art static DGNs as a discretized representation of non-dissipative ODEs tailored for C-TDGs, mirroring previous approaches in the static case~\cite{GDE, gravina2023antisymmetric}.
To the best of our knowledge, \adgn is the first framework to effectively address the problem of long-range propagation in C-TDGs and the first to bridge the gap between ODEs and C-TDGs.

The key contributions of this work can be summarized as follows: (i) We introduce the problem of long-range propagation (\ie non-dissipativeness) within C-TDGs; (ii) We introduce \adgn, a new deep graph network for learning C-TDGs based on ODEs, which enables stable and non-dissipative propagation to preserve long term dependencies in the information flow, and it does so in a theoretically founded way; (iii) We present novel benchmark datasets specifically designed to assess the ability of DGNs to propagate information over long spatio-temporal distances within C-TDGs; (iv) We conduct extensive experiments to demonstrate the benefits of our method, showing that \adgn not only outperforms state-of-the-art DGNs on synthetic long-range tasks but also outperforms them on several real-world benchmark datasets.

\section{Preliminaries} 
We consider a \textit{dynamic graph} as the tuple $\mathcal{G}(t)=(\mathcal{V}(t), \mathcal{E}(t), \mathbf{X}(t), \mathbf{E}(t))$, defined for any time $t\geq0$, which models a dynamical system of interacting entities (also known as \emph{nodes}) where interactions (or \emph{edges}) evolve over time, \ie they are dynamic in nature. Here, $\mathcal{V}(t)$ is the set of nodes that are present in the graph at time $t$, and $\mathcal{E}(t) \subseteq \{\{u,v, t^-\} \, | \, u,v \in \mathcal{V}(t), t^- < t\}$ defines the edges between them, with $t^-$ the time in which the edge $\{u,v\}$ was created. Matrices $\mathbf{X}(t)\in \mathbb{R}^{|\mathcal{V}(t)|\times d_n}$ and $\mathbf{E}(t)\in \mathbb{R}^{|\mathcal{E}(t)|\times d_e}$ contain node and edge features, respectively. The $u$-th row of $\mathbf{X}(t)$ is denoted as $\mathbf{x}_u$ and represents the features of the single node $u$. Similarly, we indicate 
the feature vector 
of the edge between nodes $u$ and $v$ created at time $t$ as $\mathbf{e}_{uvt}$. Each node $u$ is also associated to state $\mathbf{h}_u(t) \in \mathbf{H}(t)\in \mathbb{R}^{|\mathcal{V}(t)|\times d}$, which encodes node evolution over time $t$.

\begin{figure}[h]
\begin{center}
    \includegraphics[width=0.8\linewidth]{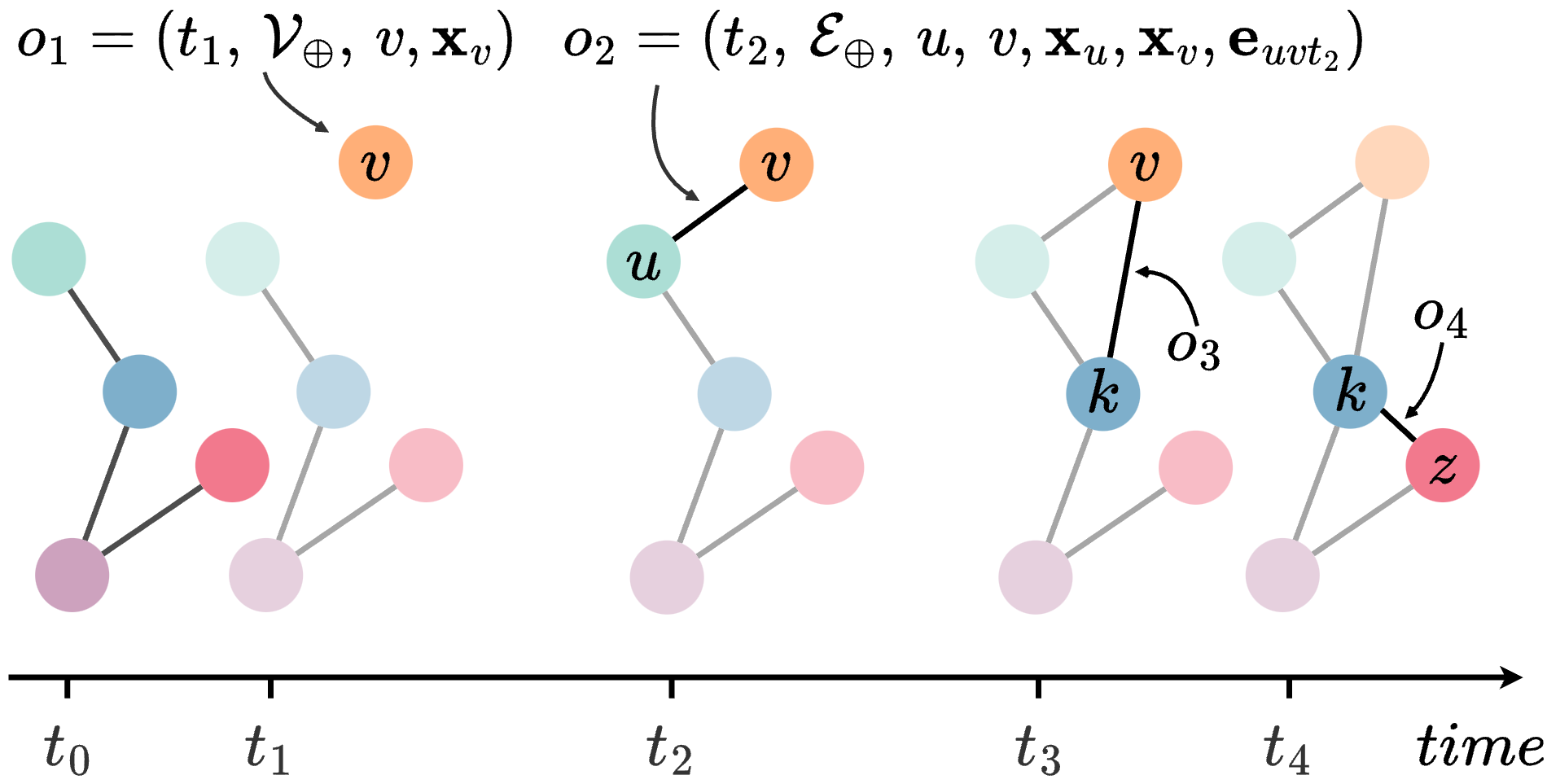}
\caption{The evolution of a Continuous-Time Dynamic Graph through the stream of events up to timestamp $t_4$. 
At each timestamp, the faded portion of the graph corresponds to historical information.}
\label{fig:ctdg}
\end{center}
\vspace{-0.3cm}
\end{figure}

In our setting, the dynamic graph is observed as a \emph{stream of events}, also known as observations, that can appear 
\emph{irregularly} over time. Therefore, the system of interacting entities is not fully observed over time, and it is known as \emph{C-TDG} \cite{ 10.1145/3184558.3191526, dynamicgraph_survey, gravina2023deep}. In this scenario, the dynamic graph can be rewritten as $\mathcal{G} = \{o_t\, | \, t\in[t_0, t_n]\}$, where each event $o_t = (t, \,EventType,\,u,\,v,\,\mathbf{x}_u,\, \mathbf{x}_v,\, \mathbf{e}_{uvt})$ is a tuple containing information regarding the timestamp, the event type, the involved nodes, and their states. The event types can be grouped into three main classes, which are node-wise events (\ie a node is updated or created), interaction events (\ie an edge is created), and deletion events (\ie a node/edge is deleted). In the following, we will refer to $\mathcal{V}_\oplus$ as the event ``node creation", and to $\mathcal{E}_\oplus$ as ``edge addition". We present in Figure~\ref{fig:ctdg} a visual exemplification of a C-TDG.

\section{Continuous-Time Graph Anti-Symmetric Network (CTAN)}\label{sec:ctan}
Learning the dynamics of a C-TDG can be cast as the problem of learning information propagation following newly observed events in the system. This entails learning a diffusion function that updates the state of node $u$ as 
\vspace{-0.25mm}
\begin{equation}\label{eq:propagation}
\mathbf{h}_u(t) = F\left(t, \mathbf{x}_u, \mathbf{h}_u(t), \{\mathbf{h}_v(t)\}, \{\mathbf{e}_{uvt^-}\}\right),
\end{equation}
where $(v,t^-)\in\mathcal{N}^t_u$, and $\mathcal{N}^t_u = \{(v, t^-) \, |\,  \{u,v,t^-\} \in \mathcal{E}(t) \}$ is the temporal neighborhood of a node $u$ at time $t$, which consists of all the historical neighbors of $u$ prior to current time $t$. In the following, we omit the time subscript from the edge feature vector to enhance readability, since it refers to a time in the past in which the edge appeared. 

In recent literature, Eq.~\ref{eq:propagation} is modeled through a dynamical system described by a learnable ordinary differential equation (ODE)~\citep{GDE, grand, pde-gcn, graphcon, gravina2023antisymmetric}. Differently from discrete models, neural-ODE-based approaches learn more effective latent dynamics and have shown the ability to learn complex temporal patterns from irregularly sampled timestamps \citep{NeuralODE, ode-irregular, neuralCDE}, making them more suitable to address C-TDG problems.

In this paper, we leverage non-dissipative ODEs~\citep{StableArchitecture, AntisymmetricRNN, gravina2023antisymmetric} for the processing of C-TDGs. Thus, we propose a framework as a solution to a \textit{stable} and \textit{non-dissipative} ODE over a streamed graph. The main goal of our work is therefore achieving preservation of long-range information between nodes over a stream of events. We do so by first showing how a generic ODE can learn the hidden dynamics of a C-TDG and then by deriving the condition under which the ODE is constrained to the desired behavior.

\myparagraph{Modeling C-TDGs}
First, we define a Cauchy problem in terms of the node-wise ODE defined in time $t\in[0,T]$
\begin{small}
    \begin{equation}
\label{eq:ode}
    \frac{\partial\mathbf{h}_u(t)}{\partial t}=f_\theta\left(t, \mathbf{x}_u, \mathbf{h}_u(t), \{\mathbf{h}_v(t)\}_{v\in\mathcal{N}^t_u}, \{\mathbf{e}_{uv}\}_{v\in\mathcal{N}^t_u}\right)
\end{equation}
\end{small}
and subject to an initial condition $\mathbf{h}_u(0) \in \mathbb{R}^d$. The term $f_\theta$ is a function parametrized by the weights $\theta$ that describes the dynamics of node state. We observe that this framework can naturally deal with events that arrive at an arbitrary time. Indeed, the original Cauchy problem in Eq.~\ref{eq:ode} can be divided into multiple sub-problems, one per each event in the C-TDG.
The $i$-th sub-problem, defined in the interval $t\in[t_s, t_e]$, is responsible for propagating only the information encoded by the $i$-th event. Overall, when a new event $o_i$ happens, the ODE in Eq.~\ref{eq:ode} computes new nodes representations $\mathbf{h}_u^i(t_e)$, starting from the initial configurations $\mathbf{h}_u^i(t_s)$. In other words, $f_\theta$ evolves the state of each node given its initial condition. The top-right of Figure~\ref{fig:method} visually summarizes this concept, showing the nodes evolution given the propagation of an incoming event.
We observe that the knowledge of past events is preserved and propagated in the system thanks to an initial condition that includes not only the current node input states but also the node representations computed in the previous sub-problem, \ie $\mathbf{h}^i_u(t_s)=\psi(\mathbf{h}_u^{i-1}(t_e), \mathbf{x}_u(i))$. We notice that the terminal time $t_e$ (treated as an hyper-parameter) is responsible for determining the extent of information propagation across the graph, since it limits the propagation to a constrained distance from the source. Consequently, smaller values of $t_e$ allow only for localized event propagation, whereas larger values enable the dissemination of information to a broader set of nodes. 

While this approach is applicable to all ODE-based DGNs for C-TDGs, we note that we are the first to introduce this truncated history propagation method in C-TDGs.

\begin{figure*}[ht]
\begin{center}
    \includegraphics[width=0.8\linewidth]{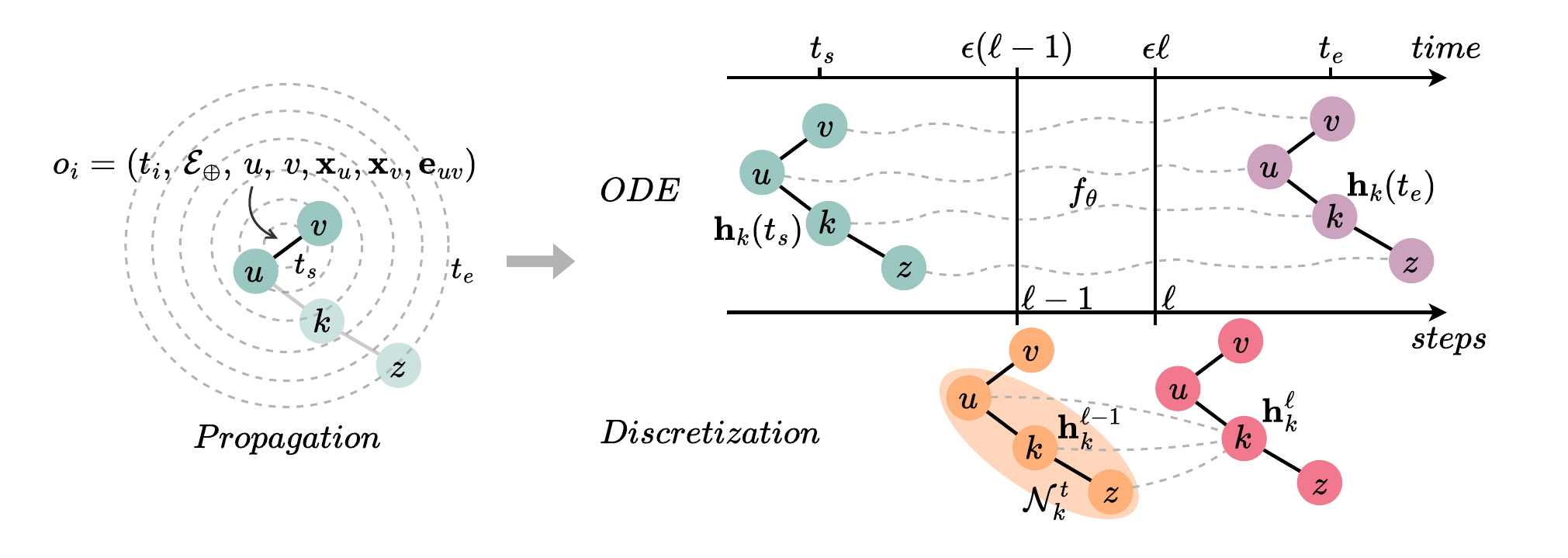}
\end{center}
\vspace{-4pt}
\caption{A high-level overview of the proposed framework illustrated for the $i$-th Cauchy sub-problem. On the left, we depict the propagation of the information of event $o_i$ through the graph. The faded portion of the graph corresponds to historical information, while the rest is the incoming event. On the right, we illustrate the evolution of node states given the propagation of the incoming event. Specifically, the top right shows the evolution as an ODE, $f_\theta$, that computes the node representation for a node $k$,  $\mathbf{h}_k(t)$. Such computation is subject to an initial condition $\mathbf{h}_k(t_s)=\psi(\mathbf{h}_k^{i-1}(t_e), \mathbf{x}_k(i))$ that includes the node representations computed in the previous sub-problem $\mathbf{h}_k^{i-1}(t_e)$ and the current node input state. In the bottom right, the discretized solution of the ODE is computed as iterative steps of the method over a discrete set of points in the time interval $[t_s, t_e]$.}
\label{fig:method}
\vspace{-2pt}
\end{figure*}

\myparagraph{Non-Dissipativeness in C-TDGs}
We now proceed to derive the condition under which the ODE is constrained to a stable and non-dissipative behavior, allowing for the propagation of long-range dependencies in the information flow. Non-Dissipativeness\footnote{The reader is referred to \citep{glendinning_1994, stabilityODE} for an in-depth analysis of dissipative dynamical systems.} in C-TDGs can be dissected into two components: non-dissipativeness over space and over time. 
\begin{definition}[\textit{Non-dissipativeness over space}]\label{def:non-dissip}
\textit{Let $u,v\in\mathcal{V}(t)$ be two nodes of the C-TDG at some time $t$, connected by a path of length $L$. If an event $o_i$ occurs at node $u$, then the information of $o_i$ is propagated from $u$ to $v$, $\forall L\geq0$.}
\end{definition}

We start by instantiating Eq.~\ref{eq:ode} as 
\begin{small}
\begin{equation}
\label{eq:our_ode}
\frac{\partial \mathbf{h}_u(t)}{\partial t} = \sigma \Bigl(\mathbf{W}_t \mathbf{h}_u(t) 
     +\Phi\left(\{\mathbf{h}_v(t),\mathbf{e}_{uv}, t_v^-, t \}_{v\in\mathcal{N}_u^t}\right) \Bigr)
\end{equation}
\end{small}
where $\sigma$ is a monotonically non-decreasing activation function; $\Phi$ is the aggregation function that computes the representation of the neighborhood of the node $u$ considering node states and edge features; $t_v^-$ is the time point of the previous event for node $v$; and  $\mathbf{W}_t\in\mathbb{R}^{d\times d}$. Here and in the following, the bias term is omitted for simplicity. We notice that including $t_v^-$ in $\Phi$ encodes the time elapsed since the previous event involving node $v$. This 
inclusion allows for smooth updates of the node’s current state during the time interval to prevent the staleness problem~\citep{dynamicgraph_survey}.

As discussed in \cite{StableArchitecture}
, a non-dissipative propagation is directly linked to the sensitivity of the solution of the ODE to its initial condition, thus to the stability of the system. Such sensitivity is controlled by the Jacobian's eigenvalues of Eq.~\ref{eq:our_ode}. Given $\lambda_i(\mathbf{J}(t))$ the $i$-th eigenvalue of the Jacobian, when 
$Re(\lambda_i(\mathbf{J}(t))) = 0$ for $i=1,...,d$ the initial condition is effectively propagated into the final node representation, making the system both stable and non-dissipative\footnote{This result holds also when the eigenvalues of the Jacobian are still bounded in a small neighborhood around the imaginary axis~\citep{gravina2023antisymmetric}.}. 

\begin{definition}[\textit{Non-dissipativeness over time}]\label{def:nd-time}
\textit{Let $u\in \mathcal{V}(t)$ be a node in the C-TDG at time $t$ and $o_t$ an event that occurs at node $u$ at time $t$. A DGN for C-TDGs is non-dissipative over time if, regardless of how many more events subsequently occur at $u$, the information of event $o_t$ will persist in $u$'s embedding.}
\end{definition}
In essence, Definition~\ref{def:nd-time} captures the idea that the embedding computed by a DGN for a node in a C-TDG retains the information from a specific event indefinitely, ensuring that the historical context is preserved and not forgotten despite the occurrence of additional events at that node.

To show the property of non-dissipativity over time, we analyze the entire system defined in Eq.~\ref{eq:our_ode} from a temporal perspective. Thus, Eq.~\ref{eq:our_ode} can be reformulated as:
\begin{small}
\begin{align}\label{eq:our_ode_time}
    \frac{\partial \mathbf{h}_u(t)}{\partial t} = \sigma \Bigl(&\mathbf{W}_t \psi(\mathbf{h}_u(t),\mathbf{x}_u(t))  \nonumber\\
    &+  \Phi\left(\{\psi(\mathbf{h}_v(t),\mathbf{x}_v(t)),\mathbf{e}_{uv}, t_v^-, t \}_{v\in\mathcal{N}_u^t}\right) \Bigr)
\end{align}
\end{small}
where $\psi$ is the function that computes the initial condition for the propagation of each event considering the node representations computed in the previous event propagation $\mathbf{h}_u(t)$ and the current node
input state $\mathbf{x}_u(t)$ (as before).

In this context, we can view the system as having $e\Delta t$ steps, where $e$ denotes the number of events and $\Delta t$ represents the propagation time of an event. Furthermore, the input state of the node $\mathbf{x}_u(t)$ is only present upon  occurrence of a new event, meaning that during the propagation of events, $\mathbf{x}_u(t)$ is set to 0. Therefore, its impact on information propagation is confined to the event's specific occurrence and does not affect each step of the propagation process.

The next proposition ensures that when the eigenvalues of the Jacobian matrix of Eq.~\ref{eq:our_ode_time} are placed only on the imaginary axes, then the ODE in Eq.~\ref{eq:our_ode_time} is non-dissipative in both space and time. Thus, we guarantee the preservation of historical context over time and the propagation of event information through the C-TDGs.
\begin{proposition}\label{prop:space_time_non_dissip}
Provided that the weight matrix $\mathbf{W_t}$ is anti-symmetric\footnote{A matrix $\mathbf{M}\in\mathbb{R}^{d\times d}$ is anti-symmetric if $\mathbf{M}^\top = -\mathbf{M}$.} and the aggregation function $\Phi$ does not depend on $\mathbf{h}_u(t)$, then the ODE in Eq.~\ref{eq:our_ode_time} is stable and non-dissipative over space and time if the resulting Jacobian matrix has purely imaginary eigenvalues, \ie $$ Re(\lambda_i(\mathbf{J}(t))) = 0, \forall i=1, ..., d.$$
\end{proposition}
For the proof, we refer the reader to \citep{gravina2023antisymmetric}(Appendix B) and substitute the Jacobian computed w.r.t. space with a Jacobian computed w.r.t. time.

By constraining weight matrix $\mathbf{W_t}$ to be anti-symmetric we obtain that the ODE in Eq.~\ref{eq:our_ode} is not-dissipative in both space and time, guaranteeing the preservation of historical node context over time while propagating event information ``spatially" through the C-TDGs. We provide a more in-depth analysis of non-dissipativeness over time in Appendix~\ref{app:temporal_analysis} where we show that varying the formulation of $\psi$ can yield to diverse behaviors.

\myparagraph{Numerical discretization} 
Now that we have defined the conditions under which the ODE in Eq.~\ref{eq:our_ode} is stable and non-dissipative, \ie it can propagate long-range dependencies between nodes in the C-TDG, we observe that computing the analytical solution of an ODE is usually infeasible. It is common practice to rely on a discretization method to compute an approximate solution by multiple applications of the method over a discrete set of points in the time interval $[t_s, t_e]$. This process is visually summarized at the bottom of Figure~\ref{fig:method}. We employ the \textit{forward Euler's method} to discretize Eq.~\ref{eq:our_ode} for the $i$-th Cauchy sub-problem, yielding the following node state update equation for the node $u$ at step $\ell$:

\vspace{-0.33cm}
\begin{small}
\begin{multline}\label{eq:discretization}
    \mathbf{h}^{\ell}_u = \mathbf{h}^{\ell-1}_u +\epsilon \sigma \Bigl((\mathbf{W}_\ell-\mathbf{W}_\ell^\top-\gamma\mathbf{I})\mathbf{h}^{\ell-1}_u \\
    +\Phi\left(\{\mathbf{h}_v(t),\mathbf{e}_{uv}, t_v^-, t \}_{v\in\mathcal{N}_u^t}\right) \Bigr),
\end{multline}
\end{small}
with $\epsilon>0$ being the discretization step size.
We notice that the anti-symmetric weight matrix $(\mathbf{W}_\ell-\mathbf{W}_\ell^\top)$ is subtracted by the term $\gamma\mathbf{I}$ to preserve the stability of the forward Euler's method, see Appendix~\ref{app:euler_stability} for a more in-depth analysis. We refer to $\mathbf{I}$ as the identity matrix and $\gamma$ to a hyper-parameter that regulates the stability of the discretized diffusion. We note that the resulting neural architecture contains as many layers as the discretization steps, \ie $L = t_e/\epsilon$.


\myparagraph{Truncated non-dissipative propagation}
As previously discussed
, the number of iterations in the discretization (i.e., the terminal time $t_e$) plays a crucial role in the propagation. Specifically, few iterations result in a localized event propagation. Consequently, the non-dissipative event propagation does not reach each node in the graph, causing a \textit{truncated} non-dissipative propagation. This method allows scaling the radius of propagation of information depending on the number of discretization steps, thus allowing for a scalable long-range propagation in C-TDGs. Crucially, we notice that, even with few discretization steps, it is still possible to propagate information from a node $u$ to $z$ (if a path of length $P$ connects $u$ and $z$). As an example, consider the situation depicted in the left segment of Figure~\ref{fig:method}, where nodes $u$ and $v$ establish a connection at some time $t$, and our objective is to transmit this information to node $z$. In this scenario, we assume $L=1$, thus the propagation is truncated before $z$. Upon the arrival of the 
event at time $t$, this is initially relayed (due to the constraint of $L=1$) to node $k$, which then captures and retains this information. If a future event at time $t+\tau$ involving node $k$ occurs, its state is propagated, ultimately reaching node $z$. Consequently, the information originating from node $u$ successfully traverses the structure to reach node $z$. 
More formally, if it exists a sequence of (at least $P/L$) successive events, such that each future $i$-th event is propagated to an intermediate node at distance $iP/L$ from $u$, then $u$ is able to directly share its information with $z$. 
Therefore, even with a limited number of discretization steps, the non-dissipative behavior enables scaling the single event propagation to cover a larger portion of the C-TDG.
We also notice that if the number of iterations is at least equal to the longest shortest path in the C-TDG, then each event is always propagated throughout the whole graph. 

\myparagraph{The CTAN framework}
We name the framework defined above \fulladgn (\adgn). Note that $\Phi$ in Eq.~\ref{eq:our_ode} and \ref{eq:discretization} can be any function that aggregates nodes and edges states. Then, \adgn can leverage the aggregation function that is more adequate for the specific task. As an exemplification of this, in Section~\ref{sec:experiments} we leverage the aggregation scheme based on the one proposed by \cite{TrasformerConv}:

\begin{small}
\begin{multline}\label{eq:aggregation}
    \Phi\left(\{\mathbf{h}_v(t),\mathbf{e}_{uv}, t_v^-, t \}_{v\in\mathcal{N}_u^t}\right) =\\
=\sum_{v \in \mathcal{N}_u^t \cup \{ u \}} \alpha_{uv} \left(\mathbf{V}_n\mathbf{h}_{v}^{\ell-1} + \mathbf{V}_e\hat{\mathbf{e}}_{uv}\right)
\end{multline}
\end{small}
where $\hat{\mathbf{e}}_{uv} = \mathbf{e}_{uv} \| \left(\mathbf{V}(t-t_v^-)\right)$ is the new edge representation computed as the concatenation between the original edge attributes and a learned embedding of the elapsed time from the previous neighbor interaction, and $\alpha_{uv}=\textrm{softmax} \left(\frac{\mathbf{q}^{\top}\mathbf{K}}{\sqrt{d}} \right)$ is the attention coefficient 
with $d$ the hidden size of each head, $\mathbf{q} = \mathbf{V}_q\mathbf{h}_u^{\ell-1}$, and $\mathbf{K} = \mathbf{V}_k\mathbf{h}_v^{\ell-1} + \mathbf{V}_e\hat{\mathbf{e}}_{uv}$. 

Despite \adgn being designed from the general perspective of layer-dependent weights, it can be used with weight sharing between layers (as in Section~\ref{sec:experiments}).

\section{Experiments}
\label{sec:experiments}

To evaluate the performance of \adgn, we design two novel temporal tasks which require propagation of long-range information by design, Section~\ref{sec:long_range} and Section~\ref{sec:pascal}.
Afterwards, we assess the performance of the proposed \adgn approach on classical benchmarks for C-TDGs in Section~\ref{sec:exp_jodie}.
We complement these classical benchmarks with a larger evaluation on the TGB framework~\cite{huang2023temporal} in Appendix~\ref{app:tgb_results}, showcasing the model capabilities in diverse settings, covering evaluations with (i) improved negative sampling techniques and (ii) transductive and inductive settings.
In Appendix~\ref{app:scalability} we conduct an investigation on the scalability property of CTAN.
In Appendix~\ref{app:datasets}, we present comprehensive descriptions and statistics of the  datasets.
We release the long-range benchmarks and the code implementing our methodology and reproducing our analysis at \url{https://github.com/gravins/non-dissipative-propagation-CTDGs}.

\myparagraph{Shared Experimental Settings}
In the following experiments, we consider weight sharing of \adgn parameters across the neural layers.  
We compare \adgn against four popular dynamic graph network methods (\ie DyRep~\citep{dyrep}, JODIE~\citep{jodie}, TGAT~\citep{TGAT}, and TGN~\citep{tgn_rossi2020}) and include recent methods GraphMixer~\cite{cong2023we} and DyGFormer~\cite{yu2023towards} for evaluation in long-range tasks. 
To ensure fair comparison and efficient implementation, we implement these methods in our framework. 
With the same purpose, we reuse the graph convolution operators in the original literature, considering for all methods the aggregation function defined in Eq.~\ref{eq:aggregation}.
We designed each model as a combination of two components: (i) the DGN (\ie \adgn or a baseline) which is responsible to compute the node representations; (ii) the readout that maps node embeddings into the output space.
The readout is a 2-layer MLP, used in all models with the same architecture. 
We perform hyper-parameter tuning via grid search, considering a fixed parameter budget based on the number of graph convolutional layers (GCLs). Specifically, for the maximum number of GCL in the grid, we select the embedding dimension so that the total number of parameters matches the budget; such embedding dimension is used across every other configuration.
We report more detailed information on each task in their respective subsections. Detailed information about hyper-parameter grids and training of models are in Appendix~\ref{app:grids}.
While we do not directly investigate the optimal terminal time $t_e$ within the hyper-parameter space, we implicitly address this aspect through the choice of the step size $\epsilon$ and the maximum number of layers $L$, as they jointly determine the terminal time, \ie $t_e=\epsilon L$.

\subsection{Long Range Tasks}\label{sec:exp_non-diss}
%
%
Here, we introduce two temporal tasks which contain long-range interaction (LRI). The first is a \textit{Sequence Classification} task on path graphs~\citep{Bondy1976} and the second an extension to the temporal domain of the classification task \textit{PascalVOC-SP} introduced in the Long Range Graph Benchmark~\citep{dwivedi2022long}.

\subsubsection{Sequence Classification on Temporal Path Graph}\label{sec:long_range}

\myparagraph{Setup}
Inspired by the tasks in~\cite{AntisymmetricRNN}, we consider a sequence classification task requiring long-range information on a temporal interpretation of a path graph~\citep{Bondy1976}.
Here, the nodes
of the path graph appear sequentially over time from first to last, \ie each event in the C-TDG connects each node to the previous one in the path graph (see Appendix~\ref{app:datasets} for a reference to our dataset).
The task objective is to predict the feature observed at the source node in the first event after having traversed the entire temporal path graph, \ie after reaching the last event in the stream.
After the model processes the last event in the graph, the output prediction for the whole graph is computed by a readout that takes as input the updated embedding of the destination node of the last event in the C-TDG.
The task requires models to propagate the information seen at the first node through the entire sequence. Models that exhibit smoothing or dissipative behavior will fail to transmit relevant information to the destination node for longer sequences, resulting in poor performance.

When creating the dataset, we set the feature of the first source node to be either 1 or -1, and we use uniformly random sampled features for intermediate nodes and edges to ensure the only task-relevant information is on the earliest node.
We forward events one at a time to update neighboring nodes representations  (\ie batch size is 1).
We considered graphs of different sizes, from length 3 to 20, to test how long information is propagated, i.e., longer graphs force models to propagate information for longer.
During training, we optimize the binary cross-entropy loss over two classes corresponding to the two possible signals (1 or -1) placed on the initial node.
Each experimental run is repeated 10 times for different weight initializations; the grid is computed considering a budget of $\sim$20k trainable parameters.
The best performing configuration is chosen based on the validation loss.
Appendix~\ref{app:grids} reports more training details and the grid of hyper-parameters.

\myparagraph{Results}
The test accuracy on the sequence classification task is in Table~\ref{tab:results-classif-path-graph_small} (comprehensive results are reported in Appendix~\ref{app:complete_results}).
\adgn exhibits exceptional performance in comparison to reference state-of-the-art methods.
This result highlights the capability of our method to propagate information seen on the first node throughout long paths. 
Meanwhile, several baseline models struggle in solving such a task because the information is lost through the time-steps:
in practice, informative gradients vanish over time.
Note that, memory-less methods such as TGAT, GraphMixer and DyGformer can not effectively propagate information past the number of layers (\ie hops) used in the neighbor aggregation. Note that while the latter two methods are designed for 1-hop aggregation, TGAT allows for variable number of GCLs aggregations, which we test up to 5. We notice TGAT can solve the task at distance 5, but fails for longer graphs.
JODIE and TGN are memory-based methods, which grants them the ability to solve tasks for longer distances, but being RNN-based methods inherently struggle to maintain long-term dependencies~\citep{279181,AntisymmetricRNN}. TGN fails at distance 7, while JODIE at distance 15.
\adgn on the other hand, better propagates information for longer distances, solving the task even at length 20.
\vspace{-0.25cm}
\begin{table}[h]
\centering
\caption{Results of the sequence classification on path graph long-range task, for increasing graph length~$n$. The performance metric is the mean test set accuracy score, averaged over 10 different random weights initializations for each model configuration. 
\label{tab:results-classif-path-graph_small}\\}

\footnotesize

\resizebox{0.48\textwidth}{!}
{
\setlength{\tabcolsep}{2pt} 
\begin{tabular}{lcccc} 
\hline\toprule 
     & $n$=3 
     & $n$=9 
     & $n$=15 
     & $n$=20 \\
\midrule
DyGFormer    &      \one{100.0$_{\pm0.0}$} 
            &      53.02$_{\pm6.06}$ 
            &       42.80$_{\pm16.25}$ 
            &      42.79$_{\pm19.62}$  \\
DyRep        &       \one{100.0$_{\pm0.0}$} 
            &      47.93$_{\pm2.73}$ 
            &        48.60$_{\pm2.48}$ 
            &       50.47$_{\pm2.88}$ \\
GraphMixer   &       \one{100.0$_{\pm0.0}$} 
            &       52.80$_{\pm5.56}$ 
            &       52.49$_{\pm5.36}$ 
            &        52.04$_{\pm8.20}$ \\
JODIE        &       \one{100.0$_{\pm0.0}$} 
            &       \one{100.0$_{\pm0.0}$} 
            &       60.0$_{\pm14.91}$ 
            &       50.87$_{\pm2.46}$ \\
TGAT         &       \one{100.0$_{\pm0.0}$} 
            &      47.87$_{\pm2.72}$ 
            &       50.53$_{\pm2.15}$ 
            &       49.07$_{\pm1.55}$ \\
TGN          &       \one{100.0$_{\pm0.0}$} 
            &      48.13$_{\pm1.63}$ 
            &       48.67$_{\pm2.76}$ 
            &       50.13$_{\pm2.17}$ \vspace{3pt}\\
Our    &       \one{100.0$_{\pm0.0}$} 
        &      99.93$_{\pm0.21}$ 
        &       \one{93.47$_{\pm8.78}$} 
        &      \one{88.93$_{\pm12.06}$} \\
\bottomrule\hline 
\end{tabular}
}
\end{table}

\subsubsection{Classification on Temporal Pascal-VOC}\label{sec:pascal}
\myparagraph{Setup}
We consider edge classification on a temporal interpretation of the \texttt{PascalVOC-SP} dataset, which has been previously employed by \cite{dwivedi2022long} as a benchmark to show the efficacy of capturing LRI in static graphs. 
Here, we adapt the task to the C-TDG domain: we forward edges one at a time and predict the class of the destination node.
We generate temporal graphs starting from the dataset of \texttt{rag-boundary} graphs extracted from Pascal VOC 2011 provided in~\cite{dwivedi2022long} (more details are provided in Appendix~\ref{app:datasets}).
We consider two degrees of SLIC superpixels compactness, \ie 10 and 30. Larger compactness means more patches, with less information included in each patch and more to be propagated.

During training, we optimize for the F1-score as in \cite{dwivedi2022long}. 
To benchmark the ability of models to propagate information through the graph, we test model performance for an increasing number of GCLs. 
Fewer GCLs require models to store and transmit relevant information along node embeddings rather than relying on effectively aggregating information from increasingly larger neighborhoods.
Each experimental run is repeated 5 times for multiple weight initializations.
The hyper-parameter grid is computed considering a budget of trainable parameters per
model equal to $\sim$40k.
Appendix~\ref{app:grids} provides further training and model selection details. 

\myparagraph{Results}
Table~\ref{tab:pascal_results} reports the average F1-score on the temporal PascalVOC-SP task. 
Note that DyRep, JODIE, GraphMixer and DyGFormer, in their original definition, do not support a variable number of GCLs, hence the results of such models are presented in the table under ``1 GCL" for clarity.
\adgn largely outperforms reference methods
. 
We observe that for SLIC compactness equal to 30, \adgn achieves a 65\% and 16\% improvement against the second best performing model (\ie TGAT), for one and three GCLs, respectively.
Interestingly, TGAT almost matches the performance of \adgn when considering five GCLs.
\begin{table*}[htb]
\centering
\caption{Results of the classification on the Temporal PascalVOC task, for increasing number of GCLs. The performance metric is the mean test set F1-score, averaged over 5 different random weights initializations for each model configuration. 
\\\label{tab:pascal_results}}

\footnotesize
\setlength{\tabcolsep}{6pt} 
\resizebox{0.75\textwidth}{!}
{
\begin{tabular}{l|ccc|ccc}
\hline\toprule
 & \multicolumn{3}{c|}{\textbf{Temporal Pascal VOC (sc=10)}} & \multicolumn{3}{c}{\textbf{Temporal Pascal VOC (sc=30)}}\\
no. GCLs & 1 & 3 & 5 & 1 & 3 & 5\\
\midrule

DyGFormer          & \one{8.45$_{\pm0.13}$} & $-$ & $-$ & 8.07$_{\pm0.27}$ & $-$  &   $-$  \\
DyRep        & 5.29$_{\pm0.47}$ & $-$ & $-$ & 5.23$_{\pm0.11}$ & $-$  &  $-$   \\
GraphMixer          & 6.60$_{\pm0.11}$ & $-$ & $-$ & 5.88$_{\pm0.08}$ & $-$  &   $-$  \\
JODIE        & 6.33$_{\pm0.41}$ & $-$ & $-$ & 5.76$_{\pm0.35}$ & $-$ & $-$   \\
TGAT         & 5.39$_{\pm0.19}$ & 6.53$_{\pm0.58}$ & 8.23$_{\pm0.73}$ & 6.04$_{\pm0.26}$ & 8.79$_{\pm0.29}$  &  10.38$_{\pm0.7}$   \\
TGN          & 6.04$_{\pm0.27}$ & 6.55$_{\pm0.46}$ &  7.51$_{\pm0.80}$ & 5.59$_{\pm0.24}$ & 7.26$_{\pm0.82}$  &   7.90$_{\pm1.31}$   \vspace{3pt}\\

Our & 7.89$_{\pm0.33}$ & \one{8.53$_{\pm1.06}$} & \one{8.88$_{\pm0.98}$} & \one{9.98$_{\pm0.33}$} & \one{10.16$_{\pm0.52}$} & \one{10.41$_{\pm0.52}$} \\
\bottomrule\hline

\end{tabular}
}
\end{table*}
This is in line with the excellent results of computationally expensive Transformers-based models in the static case~\citep{dwivedi2022long}, corroborating the advantages of self-attention blocks in modeling long-range dependencies between far away nodes.
This result also suggests that the majority of the relevant information necessary to solve the temporal Pascal VOC task may lie within neighborhoods five hops away.
We note that at SLIC compactness 10, DyGFormer benefits from the shorter long-range propagation (when sc=10 the graph contains fewer patches, hence fewer nodes and spatially closer relevant information compared to sc=30), and from its deeper architecture compared to \adgn's single-layer design, when considering the same number of spatial hops. In fact, in this setting DyGFormer contains two transformer blocks, while \adgn does not. However, we observe that by including multiple layers of \adgn (i.e., $\text{no.GCLs}>1$), our method effectively propagates information and outperforms DyGFormer even in the sc=10 task.
Nevertheless, the results indicate how \adgn is capable of propagating relevant information across the time-steps to achieve accurate predictions, even when the model is only allowed to extract information from limited, very local neighborhoods.

\subsection{Future Link Prediction Task}\label{sec:exp_jodie}
\myparagraph{Setup} 
For the C-TDG benchmarks we consider four well-known datasets proposed by \cite{jodie}, Wikipedia, Reddit, LastFM, and MOOC, to assess the model performance in real-world setting, with the task of future link prediction~\cite{jodie}.
We perform hyper-parameter tuning via grid search by optimizing the area under the roc curve (AUC). 
Results for the best configuration are provided as average on 5 random weight initializations.
To give models a fair setting for comparison, the grid is computed considering a budget of $\sim$140k trainable parameters per model and the neighbor sampler size is set to 5.
Appendix~\ref{app:grids} provides additional experimental details.

\myparagraph{Results}
Table~\ref{tab:fair_results} reports the average test AUC on the C-TDG benchmarks. \adgn shows remarkable performance, ranking first across datasets.
Our method achieves a score that on average is 4.7\% better than other baselines. 
This finding shows the importance of a non-dissipative behavior of the method even on real-world tasks, since more information need to be retained and propagated from the past to improve the final performance.
Our results demonstrate that \adgn is able to better capture and exploit such information.
Nevertheless, note that not all real-world datasets inherently present long-range dependencies. To evaluate how \adgn fares against state-of-the-art methods on several datasets, we complement this analysis with an evaluation on the TGB Benchmark~\cite{huang2023temporal}, see Appendix~\ref{app:tgb_results}. In this setting, \adgn characterizes by the best performing behaviour when considering the combination of TGB datasets. 
\begin{table}[ht]
\centering
\caption{Results of the future link prediction task. We report the mean test set AUC and std in percent averaged over 5 random weight initializations. 
\\\label{tab:fair_results}}
\footnotesize
\setlength{\tabcolsep}{5pt} 
\begin{tabular}{lcccc}
\hline\toprule
& \textbf{Wikipedia} & \textbf{Reddit} & \textbf{LastFM} & \textbf{MOOC}\\\midrule 


DyRep & 88.64$_{\pm0.15}$ 
      & 97.51$_{\pm0.10}$ 
      & 77.89$_{\pm1.39}$ 
      & 81.87$_{\pm2.47}$ 
      \\
JODIE & 94.68$_{\pm1.05}$ 
      & 96.34$_{\pm0.83}$ 
      & 69.76$_{\pm2.74}$ 
      & 81.90$_{\pm9.03}$ 
      \\
TGAT & 94.91$_{\pm0.25}$ 
     & 98.18$_{\pm0.05}$ 
     & 81.53$_{\pm0.34}$ 
     & 87.61$_{\pm0.15}$ 
     \\
TGN & 95.60$_{\pm0.18}$ 
    & 98.23$_{\pm0.10}$ 
    & 79.18$_{\pm0.79}$ 
    & 90.74$_{\pm0.99}$
\vspace{2pt}\\
Our & \one{97.55$_{\pm0.09}$} 
    & \one{98.61$_{\pm0.04}$} 
    & \one{83.81$_{\pm0.92}$} 
    & \one{92.47$_{\pm0.78}$}\\ 
\bottomrule\hline 
\end{tabular}

\end{table}


\section{Related Work}

\myparagraph{Deep Graph Network for C-TDGs}
Nowadays, most of the DGNs tailored for learning C-TDGs can be generalized within the Temporal Graph Network (TGN) framework \citep{tgn_rossi2020}.
This architecture comprises three main modules: a message module, which is responsible for computing a message that encodes the incoming event; a memory module, which stores the node histories; and a graph propagation module, which aggregates information from the local neighborhood to produce the final node representation.
Usually, the memory module is implemented as a Recurrent Neural Network (RNN) and the graph propagation module as a DGN for the processing of static graphs.
Many state-of-the-art architectures~\citep{jodie, dyrep, TGAT, streamgnn, pint} fit this framework, with later methods outperforming earlier ones thanks to advances in the local message passing part or even in the encoding of positional features.
Two recent methods~\cite{cong2023we, yu2023towards} focus on modeling long-range (time) dependencies by including longer node histories in the context while not relying on memory modules.
While recent methods often provide improved results, none of them explicitly models long-range \textit{temporal and spatial } dependencies between nodes or events in the C-TDG. 
As increasingly evidenced both in sequence-model architectures~\citep{AntisymmetricRNN}, and in the static graph case~\citep{dwivedi2022long}, propagating information across various time steps is extremely beneficial for learning
.

\adgn, instead, provably enables effective long-range propagation by design.
Note that our approach does not require the co-existence of memory \textit{and} graph propagation module, as in the TGN framework. \adgn stores all necessary information within the node embeddings themselves as computed by the graph convolution, while achieving non-dissipative propagation by design.
This makes \adgn more lightweight. 
Lastly, as TGN allows for different graph propagation modules, the general formulation of the aggregation function $\Phi$ in Eq.~\ref{eq:discretization} allows 
extending state-of-the-art DGNs for static graphs to the domain of C-TDGs through the lens of non-dissipative and stable ODEs.   

\myparagraph{Continuous Dynamic Models}
Neural Differential Equations have emerged as a class of neural networks suitable for learning continuous dynamics of systems.  \cite{NeuralODE} and \cite{AntisymmetricRNN} parameterize the continuous dynamic of RNNs through an ordinary differential equation. Similarly, \cite{EulerSN} draws a connection with Reservoir Computing. Despite the similarity with RNNs, such architectures have shown the ability to naturally incorporate data that arrive at arbitrary times \citep{NeuralODE, ode-irregular}. 
Inspired by the NeuralODE approach, GDE~\citep{GDE} links DGNs for static graphs with ODEs. In this scenario, the inter-layer dynamic of DGN's node representation is designed as a continuous information processing system defined by an ODE, which, starting from the input configuration of the nodes’ states, computes the final node representations. In the static graph domain, ODE-based architectures have been proposed with different aims, such as reducing the computational complexity of message passing~\citep{sgc,dgc}, or mitigating the over-smoothing phenomena~\citep{pde-gcn, graphcon}.

\citet{gravina2023antisymmetric} proposes A-DGN, an ODE-based model achieving non-dissipative propagation through static graphs, i.e., in the time-unaware spatial domain. We note that time-aware nodes and edges combined with possibly irregularly sampled repetitive edges between the same pair of nodes natively render A-DGN (as well as other methods designed for static graphs) inapplicable to C-TDGs. 
Less trivially, non-dissipative propagation in C-TDGs cannot be achieved through mere non-dissipative propagation through space. On the contrary, non-dissipative propagation of information through time is a property unique to DGNs designed for C-TDG, necessary for their overall non-dissipativeness
. 

To the best of our knowledge, we are the first to propose an ODE-based architecture suitable for C-TDGs that can effectively propagate long-range information between nodes.

\section{Conclusion}
We presented \fulladgn (\adgn), a new framework based on stable and non-dissipative ODEs for learning long-range interactions in Continuous-Time Dynamic Graphs (C-TDGs). Differently from previous approaches, \adgn's formulation allows scaling the radius of effective propagation of information in C-TDGs (\ie allowing for a scalable long-range propagation in C-TDGs) and reimagines state-of-the-art static DGNs as a discretization of non-dissipative ODEs for C-TDGs. To the best of our knowledge, \adgn is the first framework to address the long-range propagation problem in C-TDGs, while bridging the gap between ODEs and C-TDGs.

Our experimental investigation reveals, at first, that when it comes to capturing long-range dependencies in a task, our framework significantly surpasses state-of-the-art DGNs for C-TDGs. Our experiments indicate that \adgn is capable of propagating relevant information incrementally across time to achieve accurate predictions, even when the model is only allowed to extract information from very local neighborhoods, \ie by using only a single or few layers. Thus, \adgn enables scaling the extent of information propagation in C-TDG data structures without increasing the number of layers nor incurring in dissipative behaviors.
Moreover, our results indicate that \adgn is effective across various graph benchmarks in both real and synthetic scenarios. In essence, \adgn showcased its ability to explore long-range dependencies (even with limited resources), suggesting its potential in mitigating over-squashing in C-TDGs.

We believe that \adgn lays down the basis for further investigations of the problem of over-squashing and long-range interaction learning in the C-TDG domain. Looking ahead to future developments, we plan to extend this study to explore alternative architectures resulting from different discretization methods, such as adaptive multi-step schemes~\citep{stableEuler2}. Additionally, we aim to assess the framework's impact in the realm of efficient neural networks, such as in Reservoir Computing~\citep{reservoir_comp_book}.

\section*{Impact Statement}
This paper aims to contribute to the field of Machine Learning, specifically focusing on advancing Deep Graph Networks (DGNs) in the Continuous-Time Dynamic Graph (C-TDG) setting. The research presented herein has the potential to positively impact the ongoing exploration and applications of DGNs designed for C-TDGs. As far as we are aware, our work does not raise any ethical issues.

\section*{Acknowledgments}
The authors would like to thank Michele Russo, Huawei Technologies, Munich, Germany, and Jakub Reha, University of Amsterdam, Amsterdam, Netherlands, for the insightful discussions throughout the development of this work. The work has been partially supported by EU-EIC EMERGE (Grant No. 101070918).

\bibliographystyle{icml2024}

\newpage
\appendix
\onecolumn

\section{Non-dissipativeness Over Time}\label{app:temporal_analysis}
We note that the non-dissipative behavior of the system in Eq.~\ref{eq:our_ode_time} is contingent on the specific definition of the function $\psi$. Varying the formulation of $\psi$ can yield to diverse behaviors, significantly impacting the system's ability to either preserve or dissipate information over time.

\begin{proposition}\label{prop:non-diss_space_and_time}
Provided that the aggregation function $\Phi$ does not depend on $\mathbf{h}_u(t)$, the Jacobian matrix resulting from the ODE in Eq.~\ref{eq:our_ode_time} has purely imaginary eigenvalues, \ie $Re(\lambda_i(\mathbf{J}(t))) = 0, \forall i=1, ..., d$ if the function $\psi$ is implemented as one of the following functions:
\begin{itemize}
    \item addition, \ie $\psi=\mathbf{h}_u(t)+\mathbf{x}_u(t)$;
    \item concatenation, \ie $\psi=\mathbf{h}_u(t)\|\mathbf{x}_u(t)$;
    \item composition of \textit{tanh} and concatenation, i.e., $\psi=tanh(\mathbf{h}_u(t)\|\mathbf{x}_u(t))$.
\end{itemize}
\end{proposition}

\begin{proof}
Let's consider $\psi=\mathbf{h}_u(t)+\mathbf{x}_u(t)$, \ie \textbf{addition}. In this case Eq.~\ref{eq:our_ode_time} can be reformulated as 
\begin{equation}\label{eq:addition_stace_and_time}
    \frac{\partial \mathbf{h}_u(t)}{\partial t} = \sigma \Bigl(\mathbf{W}_t\mathbf{h}_u(t) + \mathbf{W}_t\mathbf{x}_u(t) +  \Phi\left(\{(\mathbf{h}_u(t)+\mathbf{x}_u(t)),\mathbf{e}_{uv}, t_v^-, t \}_{v\in\mathcal{N}_u^t}\right) \Bigr). 
\end{equation}
The Jacobian matrix of Eq.~\ref{eq:addition_stace_and_time} is defined as 
\begin{equation}
    \mathbf{J}(t) = \text{diag}\left[\sigma^\prime \Bigl(\mathbf{W}_t\mathbf{h}_u(t) + \mathbf{W}_t\mathbf{x}_u(t) +  \Phi\left(\{(\mathbf{h}_u(t)+\mathbf{x}_u(t)),\mathbf{e}_{uv}, t_v^-, t \}_{v\in\mathcal{N}_u^t}\right) \Bigr)\right]\mathbf{W}_t.
\end{equation}
Thus, it is the result of a matrix multiplication between invertible diagonal matrix and a weight matrix. Imposing $\mathbf{A}=\text{diag}\left[\sigma^\prime \left(\mathbf{W}_t \mathbf{h}_u(t) +  
\Phi\left(\{\mathbf{h}_v(t),\mathbf{e}_{uv}, t_v^-, t \}_{v\in\mathcal{N}_u^t}\right) + \mathbf{b}_t \right)\right]$, then the Jacobian can be rewritten as $\mathbf{J}(t)=\mathbf{A}\mathbf{W}_t$.

Let us now consider an eigenpair of $\mathbf{A} \mathbf{W}_t$, where the eigenvector is denoted by $\mathbf{v}$ and the eigenvalue by $\lambda$. Then:
\begin{align}
\label{eq:eigen}
    \mathbf{A}\mathbf{W}_t\mathbf{v} &= \lambda \mathbf{v},\notag \\
    \mathbf{W}_t\mathbf{v} &= \lambda \mathbf{A}^{-1}\mathbf{v},\notag \\
    \mathbf{v}^*\mathbf{W}_t\mathbf{v} &= \lambda (\mathbf{v}^*\mathbf{A}^{-1}\mathbf{v})
\end{align}
where $*$ represents the conjugate transpose.
On the right-hand side of Eq.~\ref{eq:eigen}, we can notice that the $(\mathbf{v}^*\mathbf{A}^{-1}\mathbf{v})$ term  is a real number. 
If the weight matrix $\mathbf{W}_t$ is anti-symmetric (\ie skew-symmetric), then it is true that  $\mathbf{W}_t^* = \mathbf{W}_t^\top=-\mathbf{W}_t$. Therefore, 
$(\mathbf{v}^*\mathbf{W}_t\mathbf{v})^* = \mathbf{v}^*\mathbf{W}_t^*\mathbf{v} = -\mathbf{v}^*\mathbf{W}_t\mathbf{v}$. Hence, 
the $\mathbf{v}^*\mathbf{W}_t\mathbf{v}$ term on the left-hand side of Eq.~\ref{eq:eigen} is an imaginary number.
Thereby, $\lambda$ needs to be purely imaginary, and, as a result, all eigenvalues of $\mathbf{J}(t)$ are purely imaginary.

Let's now consider $\psi=\mathbf{h}_u(t)\|\mathbf{x}_u(t)$, \ie \textbf{concatenation}. In this case, the product $\mathbf{W}_t(\mathbf{h}_u(t)\|\mathbf{x}_u(t))$ can be decomposed as $\mathbf{K}_t\mathbf{h}_u(t) + \mathbf{V}_t\mathbf{x}_u(t)$, with $\mathbf{K}_t$ and $\mathbf{V}_t$ weight matrices. Similarly to the addition case, the Jacobian has purely imaginary eigenvalues.

Lastly, we consider the case of $\psi=tanh(\mathbf{h}_u(t)\|\mathbf{x}_u(t))$, \ie the \textbf{composition of \textit{tanh} and concatenation}. Here, Eq.~\ref{eq:our_ode_time} is
\begin{equation}\label{eq:our_ode_time_tanh}
\frac{\partial \mathbf{h}_u(t)}{\partial t} = \sigma \bigg(\mathbf{W}_t tanh(\mathbf{h}_u(t)) + \mathbf{V}_t tanh(\mathbf{x}_u(t))  +  \Phi\left(\{tanh(\mathbf{h}_u(t)\|\mathbf{x}_u(t)),\mathbf{e}_{uv}, t_v^-, t \}_{v\in\mathcal{N}_u^t}\right) \bigg). 
\end{equation}
The Jacobian matrix is the results of the multiplication of three matrices, i.e., $\mathbf{J}(t)=\mathbf{A}\mathbf{B}\mathbf{W}_t$, with $\mathbf{A} = \mathrm{diag}\left[\sigma' \left(\mathbf{W}_ttanh(\mathbf{h}_u(t)) + \mathbf{V}_ttanh(\mathbf{x}_u(t)) +\Phi(...) + \mathbf{b}\right)\right]$ and $\mathbf{B} = \mathrm{diag}[1-tanh^2(\mathbf{h}_u(t))]$. Thanks to the associative property of multiplication $\mathbf{J}(t)=\mathbf{A}\mathbf{B}\mathbf{W}_t=(\mathbf{A}\mathbf{B})\mathbf{W}_t= \mathbf{D}\mathbf{W}_t$, where $\mathbf{D}$ is the result of the multiplication of two diagonal matrices, thus $\mathbf{D}$ is diagonal. As detailed 
for the addition case, we can conclude that the Jacobian matrix has purely imaginary eigenvalues.
\end{proof}

As a counterexample, if $\psi=\mathbf{x}_u(t)$, Eq.~\ref{eq:our_ode_time} can result in a dissipative behavior, leading to the loss of information over time and compromising the model's ability to preserve historical context, since past node information is always discarded between new events. As a result, the function $\psi$ can function as a parameter to control the balance between the dissipative and non-dissipative behavior of CTAN.

\section{Stability of the Forward Euler's Method}\label{app:euler_stability}
Following \cite{stableEuler2}, the Euler’s forward method applied to Eq.~\ref{eq:our_ode} is considered stable when $(1+\epsilon\lambda(\mathbf{J}(t)))$ lies within the unit circle in the complex plane for all eigenvalues of the system. However, since the eigenvalues of the Jacobian matrix are exclusively imaginary, it follows that $|1+\epsilon\lambda(\mathbf{J}(t))| > 1$, thus Eq.~\ref{eq:our_ode} is unstable when solved with forward Euler's method.

To enhance the stability of the numerical discretization method, we subtract a small positive constant $\gamma>0$ from the diagonal elements of the weight matrix $\mathbf{W}$. This adjustment allows the eigenvalues of the Jacobian to possess a slightly negative real part, which positions $(1+\epsilon\lambda(\mathbf{J}(t)))$ within the unit circle and enhancing the stability of the numerical discretization method. However, as detailed in Section~\ref{sec:ctan}, since $Re(\lambda_i(\mathbf{J}(t))) < 0$, the ODE becomes slightly dissipative. In conclusion, the term $\gamma$ can serve as a parameter for balancing the dissipative and non-dissipative behavior.

\section{Summary of CTAN's Propagation Capacity}
In this section, we gather the information regarding the theoretically infinite propagation capacity of our method \adgn.
Section~\ref{sec:ctan} provides the theoretical conditions (see Proposition~\ref{prop:space_time_non_dissip}) under which CTAN is non-dissipative over space and time (Definition~\ref{def:non-dissip} and \ref{def:nd-time}), allowing for the preservation of historical node context over time while propagating event information spatially through the C-TDG. When Proposition~\ref{prop:space_time_non_dissip} is satisfied, the information propagation rate is constant independently of time, since the magnitude of $\partial\mathbf{h}(t)/\partial\mathbf{h}(0)$ is constant over time. As a consequence, theoretically, there is always information flowing within the CTDG modeled by a non-dissipative (space and time) model.

For propagation over time, Appendix~\ref{app:temporal_analysis} discusses which choices of $\psi$ guarantee non-dissipativeness over time, in this case the upper bound of the range length of propagation is theoretically infinite. See Appendix~\ref{app:temporal_analysis} for examples of dissipative and non-dissipative $\psi$ functions.

For propagation over space, Section~\ref{sec:ctan} - paragraph ``Truncated non-dissipative propagation" discusses how the terminal diffusion time $t_e$ influences the range length of propagation, which is lower bounded by the number of GCLs.

\section{Datasets Description and Statistics}\label{app:datasets}
Table~\ref{tab:stats} contains the statistics of the employed datasets. In the following, we describe the datasets and their generation.  

\begin{table}[ht]
\centering
\caption{Statistics of the datasets used in our experiments. We report the total number of nodes and edges in the dataset for the temporal path graph (\ie T-PathGraph) and temporal Pascal VOC (\ie T-PascalVOC).\label{tab:stats}\\ }

\footnotesize
\begingroup
\setlength{\tabcolsep}{5pt} 
\begin{tabular}{lccccc}
\hline\toprule
& \textbf{\# Nodes} & \textbf{\# Edges}  & \textbf{\# Edge ft.}    & \textbf{Split} & \textbf{Surprise Index} \\\midrule

T-PathGraph & 3,000-20,000 & 2,000-19,000 & 1 & 70/15/15 & 1.0\\
T-PascalVOC$_{10}$ & 2,671,704 & 2,660,352 & 14 & 70/15/15 & 1.0\\
T-PascalVOC$_{30}$ & 2,990,466 & 2,906,113 & 14 & 70/15/15 & 1.0\\
Wikipedia & 9,227 & 157,474 & 172 & 70/15/15, Chronological & 0.42\\
Reddit & 11,000 & 672,447 & 172 & 70/15/15, Chronological & 0.18 \\
LastFM & 2,000 & 1,293,103 & 2 & 70/15/15, Chronological & 0.35 \\
MOOC & 7,144 & 411,749 & 4 & 70/15/15, Chronological & 0.79\\
tgbl-wiki-v2 & 9,227 & 157,474  & 172 & 70/15/15, Chronological & 0.108\\
tgbl-review-v2 & 352,637 & 4,873,540 & - & 70/15/15, Chronological & 0.987\\
tgbl-coin-v2 & 638,486 	 & 22,809,486 & - & 70/15/15, Chronological & 0.120\\
tgbl-comment & 994,790 	 	 & 44,314,507 & - & 70/15/15, Chronological & 0.823\\
\bottomrule\hline
\end{tabular}
\endgroup
\end{table}

\myparagraph{Sequence classification on temporal path graphs} To craft a temporal long-range problem, we first introduced a sequence classification problem on path graphs~\citep{Bondy1976}, which is a simple linear graph consisting of a sequence of nodes where each node is connected to the previous one. In the temporal domain, the nodes of the path graph appear sequentially over time from first to last (e.g., bottom-to-top in Figure~\ref{fig:path-graph}).

We define the task objective as the prediction of the feature seen in the first node (colored in orange in Figure~\ref{fig:path-graph}) by making the prediction leveraging only the last node representation (colored in red in Figure~\ref{fig:path-graph}) computed at the end of the sequence, i.e., when the last event appears. Note that this task is akin to the sequence classification task designed in~\citep{AntisymmetricRNN}, with the addition of a graph convolution.
We set the feature of the first node to be either $1$ or $-1$, while we set every other node and edge feature to be sampled uniformly in the range $[-1, 1]$. 
In other words, the feature $\mathbf{x}_{u_0}$ of the first node $u_0$ contains a signal to be remembered as noise is added through the propagations steps along the graph.
Formally, we create a C-TDG: $\mathcal{G} = \{o_t\, | \, t\in[t_0, t_n]\}$, such that $$o_t=(t, \,\mathcal{E}_\oplus,\,u_t,\,u_{t+1}, \mathbf{x}_{u_t}, \mathbf{x}_{u_{t+1}}, \mathbf{e}_{u_t, u_{t+1}}),$$ 
where $\mathbf{x}_{u_0} \sim \text{Bernoulli}(0.5)$\footnote{Note that we sample 1 or -1 rather than 0 or 1 to make the problem balanced around zero.}, and $\mathbf{x}_{u_j} \sim \mathcal{U}_{[-1, 1]}, \forall j > t_0$ and $\mathbf{e}_{u_t,u_{t+1}} \sim \mathcal{U}_{[-1, 1]}, \forall t$.

\begin{figure}[t]
\begin{center}
    \includegraphics[width=0.2\linewidth]{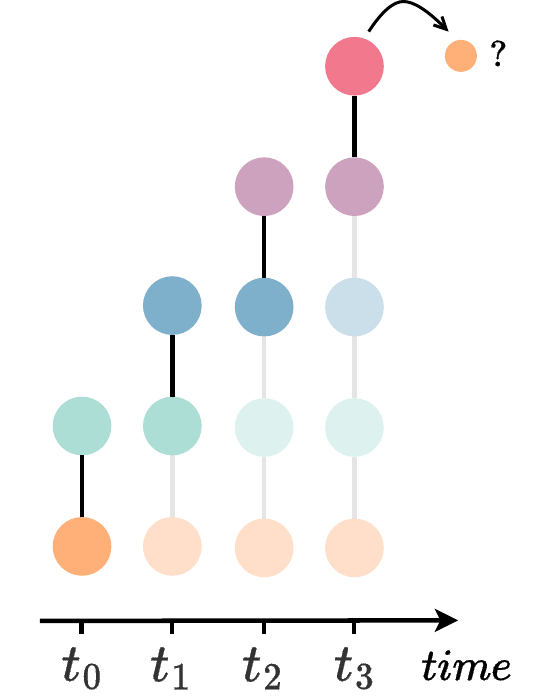}
\end{center}
\caption{The illustration of the sequence classification task on a temporal path graph consisting of 5 nodes. The first node (colored in orange) has an initial feature that can be either $1$ or $-1$. All the other nodes and edges have a feature set to random value sampled uniformly in $[-1, 1]$. At the end of the sequence, the representation computed for the last node (colored in red) is used to predict the original value of the first node. At each timestamp, the faded portion of the graph corresponds to historical information.}
\label{fig:path-graph}
\end{figure}

For this task we considered 8 temporal graph path datasets with different sizes, ranging from $n=3$ to $n=20$, with $n$ the number of nodes. For every graph size we generate 1,000 different graphs, and we split the dataset into train/val/test with the ratios 70\%-15\%-15\%.

\myparagraph{Temporal Pascal-VOC}
We use the \texttt{PascalVOC-SP}~\citep{dwivedi2022long} dataset to design a new temporal long-range task for edge classification. 
\texttt{PascalVOC-SP} is a node classification dataset composed of graphs created from the images in the Pascal VOC 2011 dataset~\citep{everingham2015pascal}. 
A graph is derived from each image by extracting superpixel nodes using the SLIC algorithm~\citep{achanta2012slic} and constructing a \texttt{rag-boundary} graph to interconnect these nodes.
Each node in a graph corresponds to one region of the image belonging to a particular class, see Figure~\ref{fig:pascalvoc-vis} for an example. 
\texttt{PascalVOC-SP} contains long-range interactions between spatially distant image patches, evidenced by its average shortest path length of 10.74 and average diameter of 27.62~\citep{dwivedi2022long}.

\begin{figure}[t]
\begin{center}
    \includegraphics[width=\linewidth]{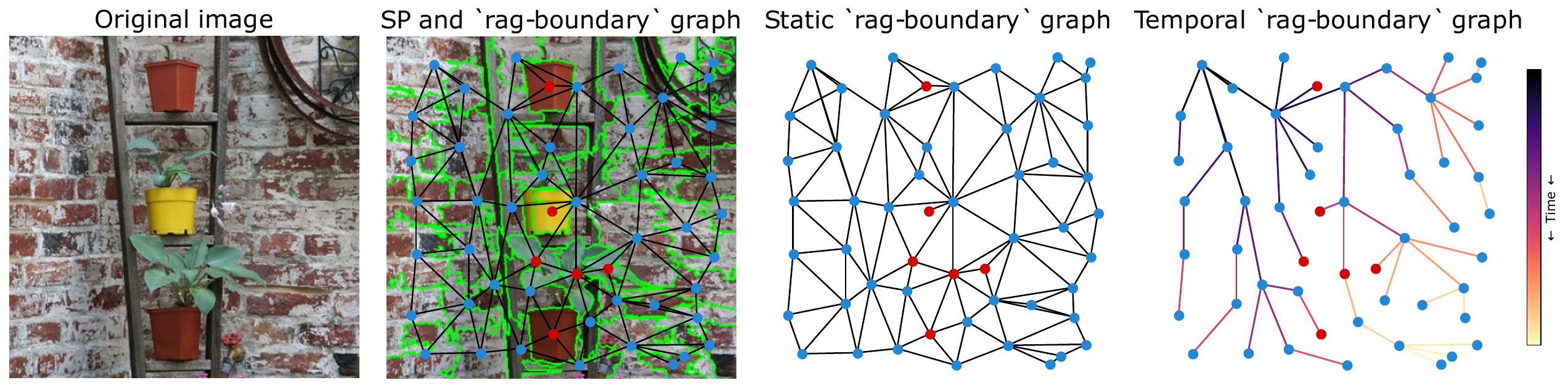}
\end{center}
\caption{Construction of the Temporal \texttt{PascalVOC-SP} dataset. The SLIC algorithm extracts patches from an image. We create the \texttt{rag-boundary} graph connecting neighboring patches based on spatial closeness. We construct a temporal graph by traversing from the topleftmost node with BFS. The goal of the task is to predict the class of the destination node at each visited edge - in the figure, either 'potted plant' (red) or 'background' (blue). For clarity in this visualization, the compactness of the SLIC algorithm is low.}
\label{fig:pascalvoc-vis}
\end{figure}

To craft a temporal task, we consider that nodes in a \texttt{rag-boundary} graph appear from the top-left to the bottom-right of the image, sequentially. 
We do so by selecting the top-leftmost node, \ie the one closest (by means of $L_1$ norm) to the origin in image coordinates.
From this node, we traverse the graph with a Breadth-First-Search, visiting each node exactly once. 
The order of edge traversal corresponds to the timestamp of edge appearance in the temporal task.
We set the task's objective to be the prediction of the class of the node that is being visited by the current edge. 
Note that the traversal removes a large number of edges from the initial graph, making the propagation of class information more difficult, see Figure~\ref{fig:pascalvoc-vis}.

Neighborhoods are constructed based on coordinates, connecting a node with its 8 spatially closest neighbors.
Nodes have 12 features extracted by channel-wise statistics on the image (mean, std, max, min) and 2 features defining the spatial location of the superpixel; we normalize these spatial features in the [0, 1] range.
We consider two SLIC superpixels compactness of 10 and 30 (smaller compactness means fewer patches).
To allow for batching, we fix the number of nodes in each graph, allowing batching of edges that occur at the same timestep across different graphs together.
To do so, we discard \texttt{rag-boundary}  graphs with fewer nodes than the limit, and discard excess nodes on graphs with more nodes than the limit, according to time (i.e., the most recent nodes are dropped).
This removes a small number of nodes corresponding to image patches on the bottom-right of the image.
In practice, for the two compactness levels 10 and 30, we set the number of minimum nodes per graph to be 434 and 474, which gives us 6,156 and 6,309 temporal graphs (out of the total 11,355 images in the dataset).
The resulting temporal datasets have 2,660,352 and 2,906,113 edges respectively.

\myparagraph{C-TDG benchmarks}
For the C-TDG benchmarks on future link prediction we consider four well-known datasets proposed by \cite{jodie}:
\begin{itemize}
    \item \textbf{Wikipedia}: one month of interactions (\ie 157,474 interactions) between user and Wikipedia pages. Specifically, it corresponds to the edits made by 8,227 users on the 1,000 most edited Wikipedia pages;
    \item \textbf{Reddit}: one month of posts (\ie interactions) made by 10,000 most active users on 1,000 most active subreddits, resulting in a total of 672,447 interactions;
    \item \textbf{LastFM}: one month of who-listens-to-which song information. The dataset consists of 1000 users and the 1000 most listened songs, resulting in 1,293,103 interactions.
    \item \textbf{MOOC}: it consists of actions done by students on a MOOC online course. The dataset contains 7,047 students (\ie users) and 98 items (\eg videos and answers), resulting in 411,749 interactions.
\end{itemize}
Since the datasets do not contain negative instances, we perform negative sampling by randomly sampling non-occurring links in the graph, as follows: (i) during training we sample negative destinations only from nodes that appear in the training set, (ii) during validation we sample them from nodes that appear in training set or validation set and (iii) during testing we sample them from the entire node set. 

For all the datasets, we considered the same chronological split into train/val/test with the ratios 70\%-15\%-15\% as proposed by \cite{TGAT}.

\myparagraph{Transductive vs Inductive Settings}
In Section~\ref{sec:exp_jodie} we employed transductive setting and random negative sampling as in~\cite{jodie, TGAT, tgn_rossi2020, yu2023towards, cong2023we}.
We chose not to employ an inductive setting as it is not easily applicable to C-TDGs. Specifically, there is no clear consensus in the literature regarding the definition of inductive settings, making it difficult to identify the nodes considered for assessing this experimental setup (e.g. \cite{TGAT} differs from \cite{tgn_rossi2020}). Some definitions of inductive settings lead to the number of sampled inductive nodes to be not statistically relevant for evaluation.
Other interpretations of inductive settings disrupt the true dynamics of the graph, i.e., in~\cite{tgn_rossi2020}, certain nodes and their associated edges are removed from the training set with the purpose of isolating an inductive set of nodes.
Thanks to the analysis performed in~\cite{yu2023towards}, we can also observe that among all the considered datasets in our paper there is mix of inductive and transductive edges, which can be measured with the \textit{surprise index} from~\cite{yu2023towards}, measuring the proportion of unseen edges at test time; reported in Table~\ref{tab:stats}.
Hence, achieving strong performance on tasks with a high surprise index offers valuable insights into the model's capability to address the inductive setting. Comparing \adgn performance to the surprise index, it is clear that \adgn can cope reasonably well even in fully inductive tasks
, such as those in Section~\ref{sec:exp_non-diss} where it generally ranks first among other baselines.

\section{Explored Hyper-Parameter Space}\label{app:grids}
In Table~\ref{tab:hyper_param_merged} we report the grids of hyper-parameters employed in our experiments by each
method. We recall that the hyper-parameters $\epsilon$, $\gamma$, and $\psi$ refer only to our method.
We used dropout only for GraphMixer and DyGFormer, where the values are loosely based on best-performing values in~\citet{yu2023empirical}.

\begin{table}[h]
\renewcommand{\arraystretch}{1.15}
\caption{The grid of hyper-parameters employed during model selection for the following three tasks: \textcolor{seqcolor}{Sequence classification on temporal path graphs}, \textcolor{pascalcolor}{Temporal Pascal-VOC}, and \textcolor{linkcolor}{Link Prediction} -- here, abbreviated as and color-coded in (\emph{Seq}, orange), (\emph{Pasc}, green), and (\emph{Link}, blue), respectively. 
For \emph{Seq} and \emph{Pasc}, we conducted \textcolor{seqcolor}{10 runs} and \textcolor{pascalcolor}{5 runs} with different random seeds for different weight initializations \emph{for each configuration}, whereas for \emph{Link}, we conducted \textcolor{linkcolor}{5 runs} only for the configuration that resulted in the best performance in the initial run. 
For the three tasks, the models were configured to have a maximum number of learnable parameters of \textcolor{seqcolor}{$\sim$20k}, \textcolor{pascalcolor}{$\sim$40k}, and \textcolor{linkcolor}{$\sim$140k}, respectively. 
Training was conducted for \textcolor{seqcolor}{20 epochs}, \textcolor{pascalcolor}{200 epochs}, and \textcolor{linkcolor}{1000 epochs}, respectively. For \emph{Seq} and \emph{Pasc}, we employed a scheduler \emph{halving the learning rate} with a patience of \textcolor{seqcolor}{5 epochs}, \textcolor{pascalcolor}{20 epochs}, respectively, whereas for \emph{Link} we used \textcolor{linkcolor}{early stopping} with a patience of \textcolor{linkcolor}{50 epochs}. 
For all tasks, the neighbor sampler size was set to 5. The batch size was set to \textcolor{seqcolor}{128}, \textcolor{pascalcolor}{256}, and \textcolor{linkcolor}{256}, respectively.
We used the \textcolor{seqcolor}{loss}, \textcolor{pascalcolor}{F1-score}, and \textcolor{linkcolor}{AUC} on the validation set to optimize for the hyper-parameters.
We used dropout only for GraphMixer and DyGFormer, where the values are loosely based on best-performing values in~\cite{yu2023empirical}.
%
\label{tab:hyper_param_merged}
\\}
\centering
\footnotesize
\begin{tabular}{l|l|ccc}
\hline\toprule
\multirow{2}{*}{\textbf{Hyper-parameter}} & \multirow{2}{*}{\textbf{Method}} & \multicolumn{3}{c}{\textbf{Values}}\\[0.6mm]
&& \textcolor{seqcolor}{Seq} & \textcolor{pascalcolor}{Pasc} & \textcolor{linkcolor}{Link} \\\midrule
\multicolumn{2}{l|}{optimizer} & \multicolumn{3}{c}{Adam} \\
\multicolumn{2}{l|}{learning rate} & \textcolor{seqcolor}{$3\cdot10^{-4}$} & \textcolor{pascalcolor}{$3\cdot10^{-4}$} & \textcolor{linkcolor}{$10^{-4}$, $10^{-5}$} \\
\multicolumn{2}{l|}{weight decay} & \textcolor{seqcolor}{$10^{-7}$} & \textcolor{pascalcolor}{$10^{-5}$} & \textcolor{linkcolor}{$10^{-6}$} \\
\multicolumn{2}{l|}{n. GCLs} & \multicolumn{3}{c}{1, 3, 5} \\
\multicolumn{2}{l|}{$\sigma$} & \multicolumn{3}{c}{tanh} \\
\multicolumn{2}{l|}{$\epsilon$} & \textcolor{seqcolor}{1, 0.5, $10^{-1}$, $10^{-2}$} & \textcolor{pascalcolor}{1, 0.5, $10^{-1}$, $10^{-2}$} & \textcolor{linkcolor}{0.5, $10^{-1}$, $10^{-2}$, $10^{-3}$} \\
\multicolumn{2}{l|}{$\gamma$} & \textcolor{seqcolor}{1, 0.5, $10^{-1}$, $10^{-2}$} & \textcolor{pascalcolor}{1, 0.5, $10^{-1}$, $10^{-2}$} & \textcolor{linkcolor}{0.5, $10^{-1}$, $10^{-2}$, $10^{-3}$} \\
\multicolumn{2}{l|}{$\psi$} & \multicolumn{3}{c}{concat, $\psi = \text{tanh}(\textbf{h}^{i-1}(t_e)||\mathbf{x}(i))$} \\
\multicolumn{2}{l|}{dropout} & \textcolor{seqcolor}{0.1, 0.2} & \textcolor{pascalcolor}{0.1, 0.2} & $-$ \\
\multicolumn{2}{l|}{time dim} & \textcolor{seqcolor}{1} & \textcolor{pascalcolor}{1} & \textcolor{linkcolor}{16} \\\midrule
\multirow{7}{*}{memory dim (= DGN dim)}& DyGFormer & \textcolor{seqcolor}{10, 5} & \textcolor{pascalcolor}{14, 7} & $-$ \\
& DyRep & \textcolor{seqcolor}{53, 26} & \textcolor{pascalcolor}{74, 37} & \textcolor{linkcolor}{118, 87} \\
& GraphMixer & \textcolor{seqcolor}{30, 15} & \textcolor{pascalcolor}{24, 12} & $-$ \\
& JODIE & \textcolor{seqcolor}{69, 34} & \textcolor{pascalcolor}{97, 48} & \textcolor{linkcolor}{164, 122} \\
& TGAT & \textcolor{seqcolor}{24, 12} & \textcolor{pascalcolor}{24, 12} & \textcolor{linkcolor}{33, 23} \\
 & TGN & \textcolor{seqcolor}{19, 9} & \textcolor{pascalcolor}{21, 10} & \textcolor{linkcolor}{33, 20} \\
& \adgn & \textcolor{seqcolor}{53, 26} & \textcolor{pascalcolor}{74, 37} & \textcolor{linkcolor}{128, 96} \\

\bottomrule\hline
\end{tabular}
\end{table}
\section{Complete Results}\label{app:complete_results}
In Table~\ref{tab:results-classif-path-graph} we report the results on the sequence classification task on temporal path graphs, and in Table~\ref{tab:fair_results_complete} we show the complete results on the link prediction task, including the performance of EdgeBank~\citep{edgebank} with different time window sizes. 
EdgeBank is  a memorization-based method without learning that simply stores previously observed edges from a fixed-size time-window from the immediate past, and predicts stored edges as positive. We evaluated EdgeBank with different time windows spanning from a size of 1\% of the training set to infinite size, \ie all observed edges are stored in memory.

In this scenario, we observe that EdgeBank is particularly good at capturing long-range information along the time dimension in the LastFM task, surpassing all the baselines and \adgn as the time window increases. We highlight that the experiments in Section~\ref{sec:exp_jodie} are meant to outline how \adgn outperforms baselines under an \textit{even field} of number of trainable parameters (i.e., 140k) and restricted range of hyper-parameter values, \eg sampler size equal to 5. On the other hand, EdgeBank is a non-parametric method that at the time of inference accesses the entire temporal adjacency matrix. In LastFM, the median node degree after training is 903 (mean $1152\pm1722$), which is high compared to other datasets. At validation time, for the average node in LastFM, EdgeBank pools information from 903 node neighbors, while the setting in Section~\ref{sec:exp_jodie} allows baselines to pool information from 5 randomly sampled neighbors. As nodes have larger degrees, sampling larger neighborhoods is fundamental to access and therefore retain information. To show that \adgn performance is limited by the considered range of hyper-parameter values, we present in Table~\ref{tab:lastfm-investigation} the performance of \adgn by solely adjusting the neighbor sampler size, while maintaining a budget of $\sim$350k learnable parameters. The evaluation involves substituting various sampler size values into the optimal combination of hyper-parameters obtained for \adgn on the LastFM dataset (as in Section~\ref{sec:exp_jodie}), with the embedding dimension configured to achieve the target of $\sim$350k learnable parameters (i.e., 192). The results indicate that \adgn performs better by adjusting the sampler size alone.

\begin{table}[]
\centering
\caption{Results of the sequence classification on path graph long-range task, for increasing graph length~$n$. The performance metric is the mean test set accuracy score, averaged over 10 different random weights initializations for each model configuration. Models have a maximum budget of learnable parameters equal to $\sim$20k.\label{tab:results-classif-path-graph}\\}

\footnotesize
\begingroup
\setlength{\tabcolsep}{3pt} 

\begin{tabular}{lcccccccc}
\hline\toprule 
 & $n$=3 & $n$=5 & $n$=7 & $n$=9 & $n$=11 & $n$=13 & $n$=15 & $n$=20 \\
\midrule
DyGFormer    &      \one{100.0$_{\pm0.0}$} &     42.55$_{\pm16.95}$ &       52.94$_{\pm7.3}$ &      53.02$_{\pm6.06}$ &        51.80$_{\pm9.52}$ &        51.70$_{\pm8.52}$ &       42.80$_{\pm16.25}$ &      42.79$_{\pm19.62}$  \\
DyRep        &       \one{100.0$_{\pm0.0}$} &        49.20$_{\pm2.10}$ &       51.00$_{\pm1.76}$ &      47.93$_{\pm2.73}$ &       44.87$_{\pm0.89}$ &       46.73$_{\pm1.55}$ &        48.60$_{\pm2.48}$ &       50.47$_{\pm2.88}$ \\
GraphMixer   &       \one{100.0$_{\pm0.0}$} &      42.58$_{\pm21.2}$ &       55.40$_{\pm6.44}$ &       52.80$_{\pm5.56}$ &      44.65$_{\pm19.42}$ &      43.77$_{\pm16.51}$ &       52.49$_{\pm5.36}$ &        52.04$_{\pm8.20}$ \\
JODIE        &       \one{100.0$_{\pm0.0}$} &       \one{100.0$_{\pm0.0}$} &       \one{100.0$_{\pm0.0}$} &       \one{100.0$_{\pm0.0}$} &       98.53$_{\pm4.64}$ &        97.4$_{\pm7.99}$ &       60.0$_{\pm14.91}$ &       50.87$_{\pm2.46}$ \\
TGAT         &       \one{100.0$_{\pm0.0}$} &       \one{100.0$_{\pm0.0}$} &      50.67$_{\pm4.12}$ &      47.87$_{\pm2.72}$ &       42.67$_{\pm2.15}$ &       43.53$_{\pm0.83}$ &       50.53$_{\pm2.15}$ &       49.07$_{\pm1.55}$ \\
TGN          &       \one{100.0$_{\pm0.0}$} &       \one{100.0$_{\pm0.0}$} &       60.20$_{\pm13.2}$ &      48.13$_{\pm1.63}$ &       45.07$_{\pm1.64}$ &        44.40$_{\pm0.64}$ &       48.67$_{\pm2.76}$ &       50.13$_{\pm2.17}$ \vspace{3pt}\\

Our    &       \one{100.0$_{\pm0.0}$} &      \one{100.0$_{\pm0.0}$} &       \one{100.0$_{\pm0.0}$} &      99.93$_{\pm0.21}$ &        \one{99.6$_{\pm0.56}$} &       \one{98.67$_{\pm1.89}$} &       \one{93.47$_{\pm8.78}$} &      \one{88.93$_{\pm12.06}$} \\
\bottomrule\hline 

\end{tabular}

\endgroup

\end{table}

\begin{table}[ht]
\centering
\caption{Mean test set AUC and std in percent averaged over 5 random weight initializations. Each model have a maximum budget of learnable weights equal to $\sim$140k. The higher, the better.\\\label{tab:fair_results_complete}}

\footnotesize
\begin{tabular}{lcccc}
\hline\toprule
& \textbf{Wikipedia} & \textbf{Reddit} & \textbf{LastFM} & \textbf{MOOC}

\\\midrule 
EdgeBank$_{1\% \,tr\, set}$         & 71.03 
                            & 71.92 
                            & 77.59 
                            & 61.29 
                            \\
EdgeBank$_{5\% \,tr\, set}$         & 81.65 
                            & 85.07 
                            & 86.75 
                            & 63.93 
                            \\
EdgeBank$_{10\% \,tr\, set}$        & 85.26 
                            & 89.07 
                            & 89.87 
                            & 65.18 
                            \\
EdgeBank$_{25\% \,tr\, set}$        & 88.31 
                            & 92.92 
                            & 92.74 
                            & 67.49 
                            \\
EdgeBank$_{50\% \,tr\, set}$        & 90.29 
                            & 94.82 
                            & 94.06 
                            & 69.63 
                            \\
EdgeBank$_{75\% \,tr\, set}$        & 91.11 
                            & 95.63 
                            & 94.55 
                            & 70.46 
                            \\
EdgeBank$_{100\% \,tr\, set}$       & 91.52 
                            & 96.08 
                            & 94.69 
                            & 70.80 
                            \\
EdgeBank$_\infty$           & 91.82 
                            & 96.42 
                            & \one{94.72} 
                            & 70.85 
                        \\\midrule 
DyRep & 88.64$_{\pm0.15}$ 
      & 97.51$_{\pm0.10}$ 
      & 77.89$_{\pm1.39}$ 
      & 81.87$_{\pm2.47}$ 
      \\
JODIE & 94.68$_{\pm1.05}$ 
      & 96.34$_{\pm0.83}$ 
      & 69.76$_{\pm2.74}$ 
      & 81.90$_{\pm9.03}$ 
      \\
TGAT & 94.91$_{\pm0.25}$ 
     & 98.18$_{\pm0.05}$ 
     & 81.53$_{\pm0.34}$ 
     & 87.61$_{\pm0.15}$ 
     \\
TGN & 95.60$_{\pm0.18}$ 
    & 98.23$_{\pm0.10}$ 
    & 79.18$_{\pm0.79}$ 
    & 90.74$_{\pm0.99}$
\vspace{3pt}\\
Our & \one{97.55$_{\pm0.09}$} 
          & \one{98.61$_{\pm0.04}$} 
          & \one{83.81$_{\pm0.92}$} 
          & \one{92.47$_{\pm0.78}$} 
          \\
\bottomrule\hline 
\end{tabular}
\end{table}

\begin{table}[h!]
\renewcommand{\arraystretch}{1.3}
\centering
\caption{Mean test set AUC and std on LastFM (in percent) for increasing size of sampled neighbors, averaged over three different weights initializations. The model has a budget of learnable weights equal to $\sim$350k. When nodes have large degrees as in LastFM, accessing larger neighborhoods with the neighbor sampler is fundamental to access and retain important information.\\\label{tab:lastfm-investigation}}
\footnotesize
\begingroup
\begin{tabular}{lcccccc}
\hline\toprule
\textbf{Sampler size} & 2 & 8 & 16 & 32 & 64 & 128 \\
\midrule
\adgn          &    82.64$_{\pm0.93}$   &  86.21$_{\pm0.58}$   &   86.16$_{\pm0.55}$   &   86.27$_{\pm0.55}$   &    86.32$_{\pm0.81}$ & \one{87.82$_{\pm0.42}$}   \\
\bottomrule\hline 
\end{tabular}
\endgroup
\end{table}

We report the average time per epoch (measured in seconds) for each model on the four considered link prediction datasets in Table~\ref{tab:time}. In this evaluation, each model has the same embedding dimension and number of GCLs. Similarly, Figure~\ref{fig:time} shows the average time per epoch of each model on the Wikipedia dataset. Here, the time is reported with respect to a varying embedding size and similar number of GCLs. We observe that our method has a speedup on average of $1.3\times$ to $2.2\times$ on the four benchmarks when one layer of graph convolutions is considered, and $1.5\times$ to $1.9\times$ when five layers are used.

\begin{table}[h!]
\centering
\caption{Mean time (in seconds) and std averaged over 10 epochs. Each model is run with an embedding dimension equal to 100 on an Intel(R) Xeon(R) Gold 6278C CPU @ 2.60GHz.\\\label{tab:time}}
\footnotesize
\begin{tabular}{rlcccc}
\hline\toprule
&    &    \textbf{Wikipedia}   &     \textbf{Reddit}     &       \textbf{LastFM}     &     \textbf{MOOC} \\\midrule
\multirow{6}{*}{1 layer}
& DyRep     & 27.07$_{\pm0.32}$       & 161.43$_{\pm0.96}$      & 216.88$_{\pm2.83}$        & 53.32$_{\pm0.56}$ \\
& JODIE     & 20.62$_{\pm0.24}$       & 131.71$_{\pm0.85}$      & 176.61$_{\pm3.02}$        & 43.92$_{\pm0.68}$ \\
& TGAT      & 11.56$_{\pm0.14}$       & 67.83$_{\pm0.64}$       & 139.79$_{\pm20.78}$       & \one{33.92$_{\pm0.50}$} \\
& TGN       & 30.92$_{\pm0.25}$       & 196.87$_{\pm1.35}$      & 289.22$_{\pm30.38}$       & 53.46$_{\pm0.62}$\vspace{3pt}\\
& Our       & \one{11.16$_{\pm0.11}$} & \one{64.48$_{\pm0.56}$} & \one{123.19$_{\pm11.33}$} & 34.42$_{\pm0.50}$ \\\midrule


\multirow{4}{*}{5 layer} &
TGAT        & 101.26$_{\pm0.46}$      & 895.35$_{\pm5.46}$       & 862.47$_{\pm217.38}$       & 73.77$_{\pm1.29}$ \\ 
& TGN       & 127.99$_{\pm0.60}$      & 1099.19$_{\pm3.91}$      & 1034.24$_{\pm221.04}$      & 95.45$_{\pm1.07}$\vspace{3pt}\\
& Our       & \one{60.16$_{\pm0.20}$} & \one{532.36$_{\pm9.87}$} & \one{495.18$_{\pm111.13}$} & \one{56.19$_{\pm0.63}$} \\


\bottomrule\hline 
\end{tabular}
\end{table}

\begin{figure}[h!]
\begin{center}
    \includegraphics[width=0.8\linewidth]{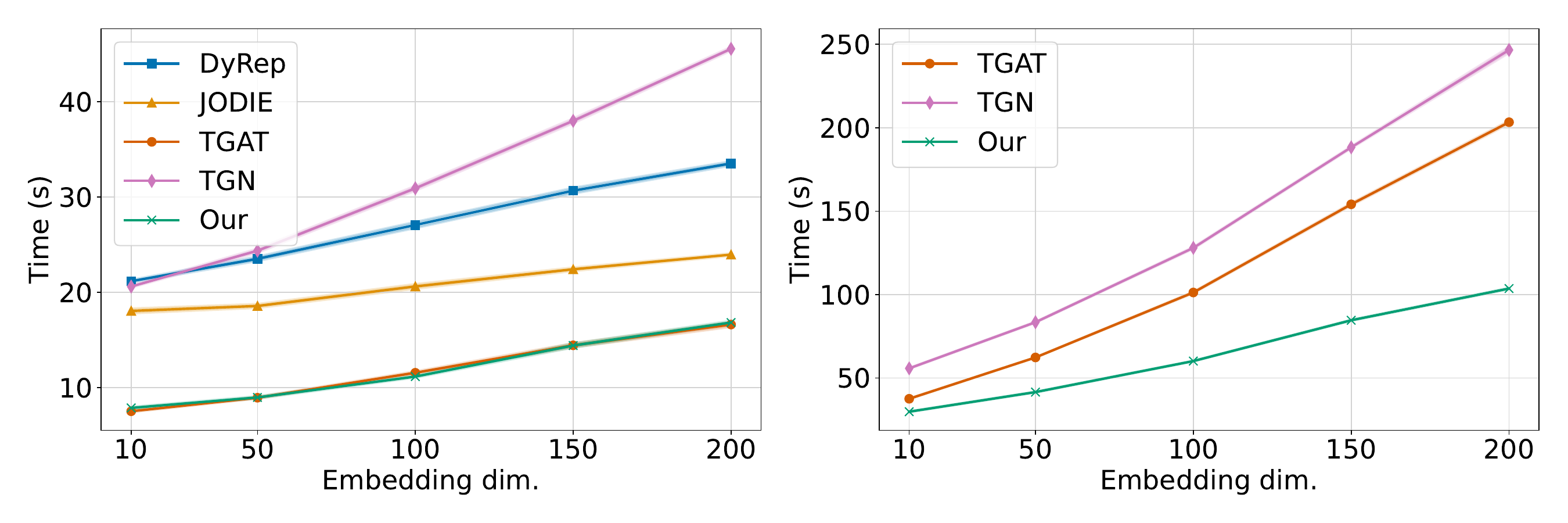}
\end{center}
\vspace{-0.3cm}
\hspace{5.2cm}(a)\hspace{6.4cm}(b)
\caption{Average time per epoch (measured in seconds) and std with respect to the embedding size computed on the Wikipedia dataset, averaged over 10 epochs. The experiments were carried out on an Intel(R) Xeon(R) Gold 6278C CPU @ 2.60GHz. On the left (a), each model has 1 DGN layer (when possible), while on the right (b) the models have 5 GCLs.}
\label{fig:time}
\end{figure}
\section{Results on TGB Benchmarks}\label{app:tgb_results}

We evaluate \adgn on the Temporal Graph Benchmark (TGB)~\citep{huang2023temporal}. 
TGB contains a set of real-world small-to-large scale benchmark datasets with varying graph properties. 
While TGB contains two different graph tasks, namely dynamic link property prediction and dynamic node property prediction, for consistency with the rest of the presentation in this paper we focus on the former.
To overcome the existing limitations on negative edge sampling, i.e., where only one random negative edge is sampled per each positive edge, TGB provides pre-sampled negative edge sets with both \textit{random} and \textit{historical} negatives~\citep{edgebank}.
Here, for each positive edge, several negatives are sampled for the  evaluation~\citep{huang2023temporal}. 
Please refer to \citep{huang2023temporal} for more information on datasets and tasks.

\myparagraph{Experimental Setup}
We adapt and re-use the TGB training loop and evaluation loop to fit our framework, and we select four datasets for measuring \adgn performance: \textit{tgbl-wiki-v2, tgbl-review-v2, tgbl-coin-v2, tgbl-comment}.
We use the configuration of \adgn, reported in Table~\ref{tab:hyper_param_tgb}.
Note that for computational efficiency, since validation passes are extremely costly given the large number of negative edges, we only do a validation pass every three training epochs.

\begin{table}[h!]
\renewcommand{\arraystretch}{1.3}
\caption{
The grid of hyper-parameters employed during model selection for \adgn on the Dynamic Link Property Prediction task on the three TGB benchmark datasets considered: \textcolor{seqcolor}{tgbl-wiki-v2}, \textcolor{pascalcolor}{tgbl-review-v2}, \textcolor{linkcolor}{tgbl-coin-v2}, \textcolor{fourth}{tgbl-comment}. For tgbl-wiki-v2 we conducted five runs with different random seeds for different weight initializations for each configuration, whereas for the other datasets we conducted three different runs.
The rest of the training configuration is taken from the TGB codebase: batch size is 200, weight decay penalty was 0, the optimized metric is Mean Reciprocal Rank and is evaluated with the TGB evaluator.\\
\label{tab:hyper_param_tgb}
}
\centering
\footnotesize
\begin{tabular}{l|cccc}
\hline\toprule
\textbf{Hyper-parameter} & \textcolor{seqcolor}{\textbf{tgbl-wiki-v2}} & \textcolor{pascalcolor}{\textbf{tgbl-review-v2}} & \textcolor{linkcolor}{\textbf{tgbl-coin-v2}} & \textcolor{fourth}{\textbf{tgbl-comment}} \\\midrule
optimizer & \multicolumn{4}{c}{Adam} \\
$\sigma$ & \multicolumn{4}{c}{tanh} \\
$\gamma$ & \textcolor{seqcolor}{0.1} & \textcolor{pascalcolor}{0.1, 0.01} & \textcolor{linkcolor}{0.1, 0.01} & \textcolor{fourth}{0.1} \\
$\psi$ & \multicolumn{4}{c}{concat, $\psi = \text{tanh}(\textbf{h}^{i-1}(t_e)||\mathbf{x}(i))$} \\
n. GCLs & \textcolor{seqcolor}{1, 2, 3} & \textcolor{pascalcolor}{1,2} & \textcolor{linkcolor}{1} & \textcolor{fourth}{1}\\
$\epsilon$ & \textcolor{seqcolor}{1.0} & \textcolor{pascalcolor}{0.5, 1.0} & \textcolor{linkcolor}{0.5, 1.0} & \textcolor{fourth}{1.0} \\
embedding dim & \multicolumn{4}{c}{256} \\
sampler size & \multicolumn{4}{c}{32} \\
learning rate &  \textcolor{seqcolor}{$10^{-3}$, $10^{-4}$, $3\cdot10^{-4}$, $3\cdot10^{-5}$} & \textcolor{pascalcolor}{$3\cdot10^{-6}$} & \textcolor{linkcolor}{$10^{-4}$} & \textcolor{fourth}{$10^{-4}$, $3\cdot10^{-4}$, $10^{-5}$, $3\cdot10^{-5}$}   \\
epochs & \textcolor{seqcolor}{200} & \textcolor{pascalcolor}{50} & \textcolor{linkcolor}{50} & \textcolor{fourth}{50} \\
LR scheduler patience & \textcolor{seqcolor}{20} & \textcolor{pascalcolor}{3} & \textcolor{linkcolor}{3} & \textcolor{fourth}{3} \\

\bottomrule\hline
\end{tabular}
\end{table}

\myparagraph{Results}
In Table~\ref{tab:tgb_results}, we report the test Mean Reciprocal Rank (MRR) for the experiments.
We note that \adgn performs quite well in general: its average rank across the four datasets is 3.25 which is the highest, together with DyGFormer~\citep{yu2023towards}.
\adgn performs quite well on tgbl-review-v2, even significantly outperforming state-of-the-art methods DyGFormer~\citep{yu2023towards} and GraphMixer~\citep{cong2023we}.
In such dataset the \textit{surprise index}~\citep{edgebank} is 0.987, meaning that nodes do not have large histories. In this case, it seems that \adgn  better propagates information from neighbors compared to methods focusing on first-hop information passing such as~\cite{cong2023we} and \cite{yu2023towards}.
On the other hand, it seems that the model in~\citep{yu2023towards} is well suited in propagating long-range \textit{time} information by modeling a large number of previous node interactions within the transformer input sequence, given enough computational budget, particularly in tgbl-wiki-v2, where nodes have long histories.
Nevertheless, we notice that even with limited number of parameters, \adgn is extremely competitive within the leaderboard.

\begin{table}[ht]
\renewcommand{\arraystretch}{1.1}
\centering
\caption{Results of the Dynamic Link Property Prediction task on the TGB benchmark datasets~\citep{huang2023temporal}. The table reports the average MRR on the test split of the datasets over the considered weight initializations. For \adgn, the average is taken over a maximum of five 
runs with different random seeds for different weight initializations. All baselines' results are taken from~\citep{yu2023empirical}. The number of parameters is computed from the TGB Baselines repository~\cite{huang2023temporal} by loading the best performing model across the model selection search.\\\label{tab:tgb_results}}
\small
\begin{tabular}{lc|cccc|c}
\hline\toprule 
&\textbf{N. params} & \textbf{tgbl-wiki-v2}  & \textbf{tgbl-review-v2}  & \textbf{tgbl-coin-v2} & \textbf{tgbl-comment}& \textbf{Avg. rank}\\\midrule
EdgeBank$_\infty$         & $-$  & 52.50         &  2.29         & 35.90 & 10.87 & 11 \\
EdgeBank$_{\text{tw-ts}}$ & $-$  & 63.25         &  2.94         & 57.36 & 12.44 & 8.25 \\
EdgeBank$_{\text{re}}$    & $-$  & 65.88         &  2.84         & 59.15 & $-$  & 8.25 \\
EdgeBank$_{\text{th}}$    & $-$  & 52.81         &  1.97         & 43.36 & $-$ & 11.33 \\\midrule
CAWN & 4M & 73.04$_{ \pm 0.60}$  & 19.30$_{ \pm 0.10}$ & $-$  & $-$ & 5.50\\
DyRep & 700k & 51.91$_{ \pm 1.95}$  & 40.06$_{ \pm 0.59}$ & 45.20$_{ \pm 4.60}$ & 28.90$_{ \pm 3.30}$ & 8.00 \\
GraphMixer & 600k & 59.75$_{ \pm 0.39}$  & 36.89$_{ \pm 1.50}$ & \one{75.57$_{ \pm 0.27}$} & \one{76.17$_{ \pm 0.17}$} & 4.25\\
DyGFormer & 1.1M & \one{79.83$_{ \pm 0.42}$}  & 22.39$_{ \pm 1.52}$ & 75.17$_{ \pm 0.38}$ & 67.03$_{ \pm 0.14}$ & \one{3.25} \\
JODIE & 200k & 63.05$_{ \pm 1.69}$  & \one{41.43$_{ \pm 0.15}$} & $-$ & $-$  & 4.50 \\
TCL & 900k & 78.11$_{ \pm 0.20}$  & 16.51$_{ \pm 1.85}$ & 68.66$_{ \pm 0.30}$ & 70.11$_{ \pm 0.83}$ & 4.25 \\
TGAT & 1.1M & 59.94$_{ \pm 1.63}$  & 19.64$_{ \pm 0.23}$ & 60.92$_{ \pm 0.57}$ & 56.20$_{ \pm 2.11}$ & 6.50 \\
TGN & 1M & 68.93$_{ \pm 0.53}$ &  37.48$_{ \pm 0.23}$ & 58.60$_{ \pm 3.70}$ & 37.90$_{ \pm 2.00}$  & 5.25  \vspace{3pt}\\
Our & 600k & 66.76$_{ \pm 0.74}$  & 40.52$_{ \pm 0.41}$ & 74.82$_{ \pm 0.42}$ & 67.10$_{ \pm 6.72}$ & \one{3.25} \\

\bottomrule\hline

\end{tabular}

\end{table}

\section{Scalability of \adgn}\label{app:scalability}

Here we conduct an investigation on the scalability property of \adgn. 
Note that while in some related works the term scalable refers to the computational complexity of methods, here we use scalable to refer to how the range of information propagation can be controlled by increasing the number of graph convolutions in \adgn.
To show this property, we show the results of the task in Section~\ref{sec:long_range} 
at different values of GCLs (when possible).
We report the results in Figure~\ref{fig:scalability-of-ctan}, which shows how for increasing GCLs, \adgn is capable of conveying information further away in the graph compared to other graph convolutional based models.
In addition, we observe that both DyGFormer and GraphMixer may have increased capability to capture long-range dependencies, however, this is only applicable to \emph{time}-only dependencies, and not spatial ones. Indeed, DyGFormer and GraphMixer model long-range time dependencies on node representations by fetching previous interactions for a node, both only relying on first-hop neighbors information and not considering spatial propagation of higher-order node information, which is in fact mentioned as a limitation of DyGFormer. Comparably, CTAN remains a graph convolution-based model, hence capable of propagating information in a non-dissipative way over time as well as over the spatial dimension of the graph, scaling the range of propagation with the number of discretization steps (equivalently, the termination time $t_e$). This property enables propagating information to neighbors beyond first-hop ones, which in turns allows solving tasks such as those in Section~\ref{sec:long_range} and \ref{sec:pascal} in the paper.

\begin{figure}[h!]
\begin{center}
    \includegraphics[width=\linewidth]{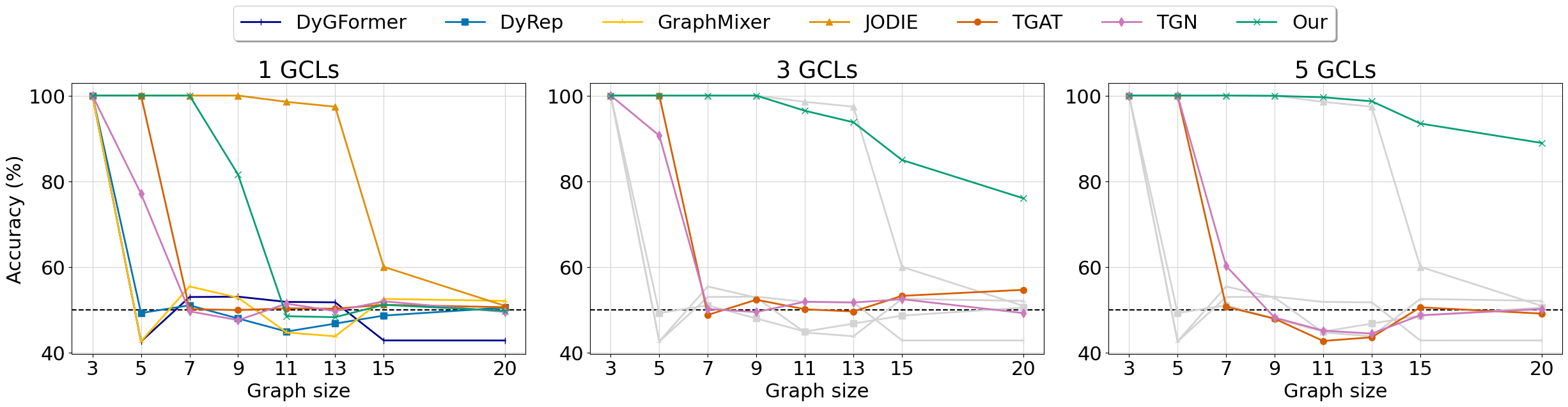}
\end{center}
\caption{Mean accuracy on the T-PathGraph task on the experiment of Section~\ref{sec:long_range}, with distinction between the performance at different number of GCLs 
(whenever possible). With 3 and 5 GCLs we report in grey the results of DyGFormer, DyRep, GraphMixer, and JODIE, which are designed for 1-hop aggregation only. The plots show that not only \adgn can better retain information at low number of GCLs, but also that increasing the number of GCL enables solving the T-PathGraph task on longer graphs, where the task is harder because information needs to be propagated further away. The number of GCL allows \adgn to \textit{scale} up the range of information propagation.}
\label{fig:scalability-of-ctan}
\end{figure}


\end{document}